\documentclass{article}

\usepackage{microtype}
\usepackage{graphicx}
\usepackage{subfigure}
\usepackage{booktabs} 

\usepackage{hyperref}

\usepackage[accepted]{icml2023}

\usepackage{amsmath}
\usepackage{amssymb}
\usepackage{mathtools}
\usepackage{amsthm}
\usepackage{xcolor,colortbl}
\newcommand{\gray}[1]{\textcolor{gray}{#1}}

\def\x{{\mathbf x}}
\def\e{{\mathbf e}}
\def\f{{\mathbf f}}
\def\h{{\mathbf h}}

\def\R{{\mathbb R}}

\def\t{{\mathbf t}}

\def\s{{\mathbf s}}
\def\F{{\mathcal F}}
\def\1{{\mathbf{1}}}

\def\C{{\mathcal C}}

\usepackage{tikz}

\newcommand{\ExternalLink}{%
    \tikz[x=1.2ex, y=1.2ex, baseline=-0.2ex]{%
        \begin{scope}[x=1ex, y=1ex]
            \clip (-0.1,-0.1) 
                --++ (-0, 1.2) 
                --++ (0.6, 0) 
                --++ (0, -0.6) 
                --++ (0.6, 0) 
                --++ (0, -1);
            \path[draw, 
                line width = 0.5, 
                rounded corners=0.5] 
                (0,0) rectangle (1,1);
        \end{scope}
        \path[draw, line width = 0.5] (0.5, 0.5) 
            -- (1, 1);
        \path[draw, line width = 0.5] (0.6, 1) 
            -- (1, 1) -- (1, 0.6);
        }
    }

\usepackage[capitalize,noabbrev]{cleveref}

\theoremstyle{plain}
\newtheorem{theorem}{Theorem}[section]
\newtheorem{proposition}[theorem]{Proposition}

\theoremstyle{definition}

\theoremstyle{remark}

\usepackage{graphicx}

\usepackage{booktabs}
\usepackage{makecell}

\usepackage{enumitem}
\usepackage{multirow}

\usepackage[textsize=tiny]{todonotes}

\icmltitlerunning{FAENet}

\begin{document}

\twocolumn[
\icmltitle{FAENet: Frame Averaging Equivariant GNN for Materials Modeling}



\icmlsetsymbol{equal}{*}

\begin{icmlauthorlist}
\icmlauthor{Alexandre Duval}{equal,cs,mila}
\icmlauthor{Victor Schmidt}{equal,mila}
\icmlauthor{Alex Hernandez Garcia}{mila}
\icmlauthor{Santiago Miret}{intel}
\icmlauthor{Fragkiskos D. Malliaros}{cs}
\icmlauthor{Yoshua Bengio}{mila,udem}
\icmlauthor{David Rolnick}{mila,mg}
\end{icmlauthorlist}

\icmlaffiliation{mila}{Mila -- Quebec AI Institute}
\icmlaffiliation{mg}{McGill Unversity}
\icmlaffiliation{intel}{Intel Labs}
\icmlaffiliation{udem}{Université de Montréal}
\icmlaffiliation{cs}{Université Paris-Saclay, CentraleSupélec, Inria}

\icmlcorrespondingauthor{Alexandre Duval, Victor Schmidt}{\texttt{\{alexandre.duval,schmidtv\}@mila.quebec}}

\icmlkeywords{equivariant graph neural network, invariant graph neural network, molecular graph, materials science, open catalyst dataset, frame averaging}

\vskip 0.3in
]



\printAffiliationsAndNotice{\icmlEqualContribution} 

\begin{abstract} 
    Applications of machine learning techniques for materials modeling typically involve functions known to be equivariant or invariant to specific symmetries. While graph neural networks (GNNs) have proven successful in such tasks, they enforce symmetries via the model architecture, which often reduces their expressivity, scalability and comprehensibility. 
    In this paper, we introduce (1) a flexible framework relying on stochastic frame-averaging (SFA) to make any model E(3)-equivariant or invariant through data transformations.
    (2) FAENet: a simple, fast and expressive GNN, optimized for SFA, that processes geometric information without any symmetry-preserving design constraints. We prove the validity of our method theoretically and empirically demonstrate its superior accuracy and computational scalability in materials modeling on the OC20 dataset (S2EF, IS2RE) as well as common molecular modeling tasks (QM9, QM7-X). A package implementation is 
    available at 
    \url{https://faenet.readthedocs.io}.
\end{abstract}

\section{Introduction}
Machine Learning (ML) methods have the ability to model complex physical and chemical interactions. It thus holds great potential for accelerating material design, which is essential to various applications such as low-carbon energy, sustainable agriculture or drug discovery. One particularly promising use case of ML is modeling the properties of complex materials systems at lower computational cost compared to expensive quantum mechanical simulation techniques like Density Functional Theory (DFT).
The heavy reliance on DFT for materials property prediction continues to impose a significant computational barrier to evaluating large number of material candidates \citep{chen2022universal}. Graph Neural Networks (GNNs) based on geometric deep learning principles have shown promise in their ability to predict a wide range of molecular properties \citep{han2022geometrically}. A key factor of the success of GNNs is their ability to leverage 3D geometric information via the representation of a collection of atoms in 3D space \citep{atz2021geometric}, which is updated based on spatial atomic interactions by passing messages between them. Another important aspect is the incorporation of geometric priors that exploit the symmetry of the data, rendering model predictions invariant or equivariant\footnote{In this work, unless specified otherwise, we consider invariance to be a special case of equivariance and will include invariance in claims regarding equivariance.} to Euclidean transformations\footnote{Rotations, reflections, and translations, which in 3D space define the group $\text{E}(3)$.}, as well as key physics principles such as the conservation of energy \citep{smidt2021euclidean}.

Symmetries and physical constraints are typically enforced directly into the model architecture, which greatly restricts the flexibility of GNNs to process geometric information \citep{gasteiger2021gemnet, fuchs2020se, satorras2021n}. As a result, these models either lack expressivity or present significantly more complex and computationally expensive architectures, as detailed in Section~\ref{sec:RW}. While state-of-the-art GNNs remain orders of magnitude faster than DFT, their inference time still limits the use of ML for downstream practically-relevant applications, which require large-scale evaluations \citep{agrawal2016fourthparadigm}. Indeed, whether we are trying to discover new drugs, new catalysts or undiscovered material systems, we need to explore exponentially vast search spaces of potential candidates \citep{bohacek1996molecularmodeling}. The above ambitions to accelerate automated material discoveries therefore require designing expressive, robust and computationally scalable models. 

To that end, we propose a novel view of 3D molecular and solid-state materials modeling, where symmetries are preserved via data projections instead of architectural constraints. Concretely, we make the following contributions: 
\begin{itemize}
    \item \textit{Symmetry-Preserving Data Augmentation via Stochastic Frame-Averaging:} We propose a flexible framework to develop equivariant GNNs for materials modeling without any architectural requirement. Building upon Frame-Averaging (FA) \citep{puny2022frame}, we propose to project data points into a canonical representation, allowing any model to be theoretically (\textit{Full} FA) or empirically (\textit{Stochastic} FA) $\text{E}(3)$-equivariant without losing expressiveness.
    %
    \item \textit{FAENet:} We introduce the Frame Averaging Equivariant Network (FAENet), a lightweight yet effective GNN whose design is not constrained by symmetry-preserving requirements. FAENet can therefore process geometric information through atom relative positions with full flexibility while rigorously preserving symmetries through the data, as leveraged by FA. 
    \item \textit{FAENet Analysis:} We verify the theoretical properties of the proposed approach, investigate its expressive power, and demonstrate its superior accuracy vs scalability trade-off compared to prior methods on four well known datasets in ML for materials science: OC20 IS2RE, S2EF (2M) for solid-state crystal structure modeling, QM7-X and QM9 for molecular modeling. 
\end{itemize}

\label{sec:intro}

\section{Related Work}

Recent works have expanded the application of ML techniques to a broad set of materials modeling tasks ranging from solid-state \citep{zitnick2020introduction, miret2022open} to molecular \citep{hoja2021qm7, ramakrishnan2014quantum} structures. 
Most existing GNN architectures have applied physics-informed 3D symmetries directly in the model architecture, making model predictions explicitly invariant or equivariant to the desired transformations\footnote{A function $f$ is invariant w.r.t. transformation $t$ if, for any input $x$, $f(t(x)) = f(x)$. On the other hand, an equivariant function is such that $f(t(x)) = t(f(x))$. Typical examples include energy and forces, respectively}.

Various GNNs are constructed to be \textbf{\text{E}(3)-invariant} by extracting invariant geometric features from atomic positions~\citep{unke2019physnet, klicpera2020directional, liu2021spherical, shuaibi2021rotation, ying2021transformers, adams2021learning}. SchNet \citep{schutt2017schnet}, leverages atom distances via a continuous convolution filter to learn a potential energy surface, making it fast but unable to distinguish between certain types of molecules. ComENet \citep{wang2022comenet} and GemNet \citep{gasteiger2021gemnet} extract additional information via bond angles and torsion angles(between quadruplets of nodes), thereby enabling the architecture to distinguish a larger set of atomic systems. Nevertheless, these methods are computationally expensive, as they use 3-hop neighbourhoods to compute torsion information for each update step at both training and inference time. Furthermore, they are usually less expressive than equivariant representations \cite{wang2022graph, joshi2022expressive}. 

\textbf{Equivariant} methods \citep{thomas2018tensor, anderson2019cormorant, fuchs2020se, brandstetter2021geometric, batatia2022mace, frank2022so3krates} focus on enforcing equivariance by using irreducible representations of the SO(3) group. These works usually combine node features (learned based on atom relative positions) with a continuous equivariant filter (constructed based on spherical harmonics and learnable radial function) via a Clebsh-Gordan tensor product to guarantee $\text{E}(3)$- or $\text{SE}(3)$-equivariant predictions. While these methods are expressive and generalize well, they can be hard to implement and very computationally expensive for training and inference.

Additionally, \citet{schutt2021equivariant, batzner20223, satorras2021n, tholke2022torchmd} modeled equivariant interactions in Cartesian space using both scalar and vector representations. These GNNs achieve good performance and are relatively fast by avoiding expensive operations. However, the authors manually design two separate sets of functions to deal with each type of representation and often use complex operations to mix their information, which renders the global architecture hard to understand. Overall, such models lack formal guarantees.

\textbf{Data augmentation} (DA) is an alternative way of incorporating the desired data symmetries without constraining the model architecture. It has been found extremely effective in computer vision. In the graph domain, \citet{hu2021forcenet} augmented molecular datasets with rotated and reflected input graphs. 
Despite showing promising results, the accuracy vs. scalability gains are not significant enough to constitute a true Pareto optimal improvement. Potential explanations include the suboptimal utilisation of geometric information given the absence of design constraints, which translates for instance into an overly complex model architecture, as well as soft and partial symmetry-preservation. 
In the 3D image domain explored by \citet{gerken2022equivariance}, equivariant DA also resulted in a mitigated performance-scalability trade-off. 
\citet{gerken2022equivariance} showed as well that DA methods matched invariant networks in accuracy for invariant tasks on 3D images, with much smaller computational cost. Despite promising results, this area remains underexplored for 3D materials property prediction tasks.


\label{sec:RW}

\section{Symmetry-Preserving Data Augmentation via Stochastic Frame Averaging}

The first part of our modeling framework builds upon the idea of frame averaging, where we map the input data to a canonical plane using Principal Component Analysis (PCA). This mapping offers a unique representation of all Euclidean transformations of the data and enables to rigorously preserve symmetries without constraining the GNN design. Frame-averaging also preserves the expressive power of the backbone architecture, ultimately leading to expressive architectures at lower inference cost. 

\subsection{Background}

Frame averaging (FA) is a framework introduced by \citet{puny2022frame} by which a function (such as a neural network) $\Phi: V \rightarrow W$ between normed vector spaces can be transformed into one that is equivariant (or invariant) with respect to a class of symmetries $G$. Suppose that $\rho_1$ and $\rho_2$ are representations of the group $G$ over $V$ and $W$, respectively. That is, for each element $g\in G$, $\rho_1(g)$ and $\rho_2(g)$ are the transformations that $g$ induces in $V$ and $W$.

A \emph{frame} is defined as a function $ \F(X):V \rightarrow 2^G$, where the following properties are relevant:
\begin{itemize}
    \item $\F$ is \emph{$G$-equivariant} if $\forall X \in V, g \in G$, $\mathcal{F}(\rho_1(g)X)=g \mathcal{F}(X)$, where $g\mathcal{F}(X)=\{ gh |h \in \mathcal{F}(X)\}$. 
    \item $\F$ is \emph{bounded} over a domain $K \subset V$ if $\exists c > 0$ such that $\forall g \in \F(X), X\in K$, $||\rho_2(g)||_{op} \leq c$, where $|| \cdot ||_{op}$ is the operator norm over $W$.
\end{itemize}
 \citet{puny2022frame} prove that if a frame is $G$-equivariant, then any arbitrary map $\Phi$ can be made equivariant (or invariant) by averaging predictions over that frame, noted $\langle\Phi\rangle_{\mathcal{F}}$:
\begin{equation}
    \langle\Phi\rangle_{\mathcal{F}}(X) = \frac{1}{|\mathcal{F}(X)|} \sum_{g \in \mathcal{F}(X)} \rho_2(g) \Phi(\rho_1(g)^{-1}X)).
    \label{eq:equiv-fa-def}
\end{equation}
The authors also show that if $\F$ is bounded, FA-based GNNs $\Phi_{\F}$ are $G$-equivariant models which preserve the expressive power of the backbone architecture, leading to maximally expressive equivariant models for learning on graphs. 

Note that these results extend those of the group averaging operator \citep{chen2021equivariant, yarotsky2022universal}, which takes the sum in Eq.~\eqref{eq:equiv-fa-def} over all group elements. FA thus offers an efficient way to achieve the same theoretical properties. 

In this work, we address the generic problem of materials property prediction where said properties are invariant/equivariant to Euclidean motions. As such, we can use FA to make our GNN $\Phi$ symmetry-preserving with respect to the group $G=\text{E}(3)$.
Note that we formalise this task as a graph/node regression problem, where we predict properties (e.g., forces $\mathbf{y} \in \mathbb{R}^{n \times 3}$ or energy $y \in \mathbb{R}$) from initial atomic configurations $(X, Z)$, where $X \in \R^{n \times 3}$ is the matrix of 3D atom positions and $Z \in \mathbb{N}^{n}$ their atomic numbers.


\subsection{Frame Construction}
\label{subsec:frame-construction}

The goal of frame construction is to find a frame uniquely determined by the configuration of atoms, such that the projected coordinates remain identical for arbitrary E(3) transformations. Additionally, the frame should be robust to slight distortions, such that similar objects are mapped to similar frames and representations. 
In order to define such a frame for our materials modeling case, we apply Principal Component Analysis (PCA) and set the principal components of the 3D atomic structure $X$ as frame axes.

Concretely, we define $\mathcal{F}(X)$ using our PCA procedure as follows: First, we compute the centroid $\mathbf{t} = \frac{1}{n} X^\top \mathbf{1} \in \R^3$ and the covariance matrix $\Sigma = (X - \mathbf{1} \mathbf{t}^\top)^\top (X - \mathbf{1} \mathbf{t}^\top)$ of atom positions $X$. Then, we solve 
$\Sigma \mathbf{u}=\lambda \mathbf{u}$ to find the eigenvectors $\mathbf{u}$ of $\Sigma$. Under the assumption of distinct eigenvalues $\lambda_1 > \lambda_2 > \lambda_3$, we use the 3$\times$3 orthogonal matrices $U = [\pm \mathbf{u}_1, \pm \mathbf{u}_2, \pm \mathbf{u}_3]$ to define the frame 
\begin{equation}
    \mathcal{F}(X) = \{(U, \mathbf{t}) | U = [\pm \mathbf{u}_1, \pm \mathbf{u}_2, \pm \mathbf{u}_3 ]\} \subset \text{E}(3).
    \label{eq:frames}
\end{equation} 
$|\mathcal{F}(X)|=2^3=8$ for 3D atomic systems. 
We apply the symmetries $g = (U,\mathbf{t}) \in \mathcal{F}(X)$ to the data by defining the following group representations: 
\begin{align}
    \label{eq:initial-group-rep}
    \rho_1(g)X&=XU^\top + \mathbf{1}\mathbf{t}^\top \\
    \rho_2(g)X&=\begin{cases} XU^\top \quad \textrm{ for equivariant predictions} \\ 
 X\qquad\;\, \textrm{ for invariant predictions}\end{cases} \notag
\end{align}


To compute model inferences, we use the frame $\mathcal{F}(X)$ defined in~Eq.~(\ref{eq:frames}) and group representations of Eq.~(\ref{eq:initial-group-rep}) to average predictions over the frame, as in~Eq.~(\ref{eq:equiv-fa-def}).

To extend this framework to periodic crystal structures, we also consider the case where the structure is not defined by just $X$ but the extended set $D=(X, Z, C, O)$, where $C \in \mathbb{R}^{3 \times 3}$ refers to the unit cell Cartesian coordinates (with directional vectors as columns) and $O \in \{-1,0,1\}^{n \times n \times 3}$ to the cell offsets between any pair of atoms given periodic boundary conditions. With this new formulation, we keep the same frame ($\F(D)=\F(X)$) but we additionally need to apply symmetries on $C$, similarly to $X$, while $O$ and $Z$ remain unchanged by any $\text{E}(3)$ transformation. For a more detailed justification, see Appendix~\ref{app:sec:SFA}.
This yields
\begin{align}
     \rho_1(g)^{-1}(D)&= \big((X - \mathbf{1}\mathbf{t}^\top)U, Z, (C - \mathbf{1}\mathbf{t}^\top)U ,O\big) \notag \\
    \rho_2(g) (D) &= \big(XU^\top, Z, CU^\top, O\big), 
    \label{eq:rho-formulations}
\end{align} 
for the equivariant case, while $\rho_2(g)$ remains the identity function for invariant predictions. 
We insert these functions into Eq.~(\ref{eq:equiv-fa-def}) to derive final model predictions. 
Note that $\rho_1^{-1}$ is computed directly since it is the quantity of interest. Since $U$ is an orthogonal matrix, we have $U^{-1}=U^\top$.

\begin{proposition}
The frame $\F$ defined in Eq.~(\ref{eq:frames}) along with transformations $\rho_1$, $\rho_2$ defined in Eq.~(\ref{eq:rho-formulations}) is bounded and G-equivariant. 
Hence, any arbitrary map $\Phi$ becomes $\text{E}(3)$-invariant/equivariant using Eq.~(\ref{eq:equiv-fa-def}) while preserving its expressive power. 
\end{proposition}

\begin{proof}
See Appendix~\ref{app:sec:SFA}.
\end{proof}

\subsection{Stochastic Frame Averaging}
\label{subsec:SFA}

The above frame averaging method, extended to crystals periodic structures and denoted \textit{Full FA} in the rest of the paper, involves the average of GNN predictions $\langle\Phi\rangle(D)$ over $|\F(X)|=8$ frames. Given our goal to enable scalable GNN training/inference, we propose an efficient approximation to \textit{Full FA} called \textit{Stochastic Frame Averaging (SFA)}.  

In SFA, we randomly sample a single symmetry $g_*=(U_*, \mathbf{t})$ uniformly from $\F$ for each data point at every forward pass. Final model predictions thus become:
\begin{align}
    \label{eq:FA_approx}
    \Phi_{\mathcal{F}}^{\text{\textit{invar}}}(D)&= \Phi \left(((X - \mathbf{1}\mathbf{t}^\top)U_*, Z, (C - \mathbf{1}\mathbf{t}^\top)U_* ,O) \right) \nonumber \\
    \Phi_{\mathcal{F}}^{\text{\textit{equiv}}}(D) &= \Phi_{\mathcal{F}}^{\text{\textit{invar}}}(D) \cdot U_*^\top,
\end{align}
for the invariant $\Phi_{\mathcal{F}}^{\text{\textit{invar}}}(D)$ and equivariant cases $\Phi_{\mathcal{F}}^{\text{\textit{equiv}}}(D)$.

While sampling only a single frame component does not theoretically guarantee exact invariance or equivariance, we empirically demonstrate in \cref{app:sec:empirical-eval} that our stochastic approximation is almost perfectly invariant and equivariant with up to 8$\times$ faster compute time compared to \textit{Full Frame Averaging}. Our approach can be viewed as a variant of data augmentation where we let the model learn symmetries by stochastically sampling among a few possible projections as opposed to all possible rotations and reflections of the data. This leads to a better preservation of data symmetries in addition to a higher performance with similar compute time when compared to conventional DA (see \cref{app:sec:empirical-eval}). 
We hence bridge the gap between data augmentation approaches, which present unleveraged benefits (e.g. no model design constraints), and hard-engineered equivariant methods (with design constraints). 

Additionally, this setup imbues our framework with valuable flexibility. SFA can (1) be incorporated into the training pipeline of any GNN; (2) enforce data symmetries approximately or exactly by using all 8 frames; (3) consider $\text{SE}(3)$ (not equivariant to reflections) instead of $\text{E}(3)$ by restricting $\mathcal{F}(X)$ to orthonormal positive orientation matrices, generally with $|\mathcal{F}(X)|=4$ elements; (4) be adapted to any dataset. For instance, the z-axis is fixed in OC20 \cite{chanussot2021open} (i.e. the catalyst is always below the adsorbate). Hence, we could consider only 2D rotation/reflection around the z-axis, which yields a total of $2$ frames for \text{SE}$(3)$ equivariance, where each has dimension $2\times2$.


\label{sec:FA}

\section{FAENet}

\begin{figure*}[t]
    \centering
    \includegraphics[width=\textwidth]{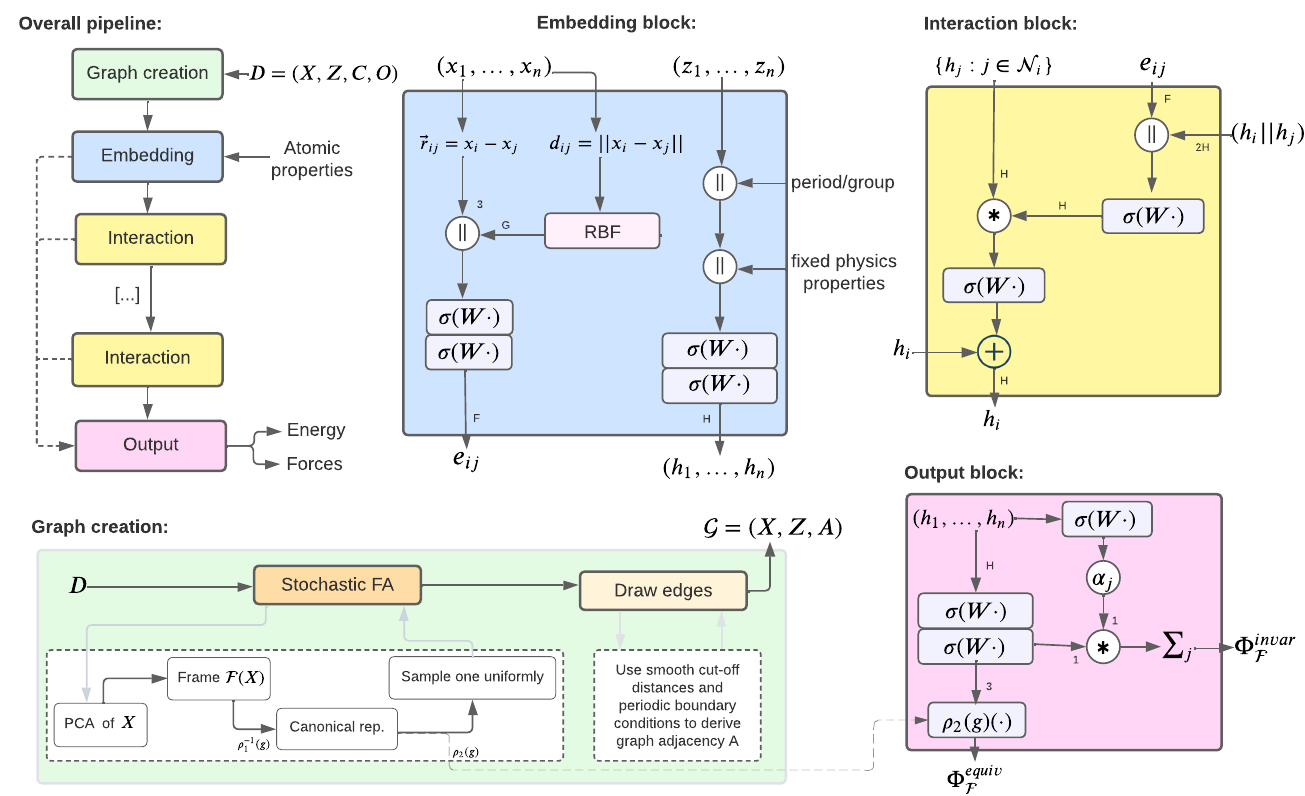}
    \caption{Overview of FAENet architecture making use of the Stochastic Frame Averaging framework. FAENet takes an input sample $D$ and makes an invariant $\Phi^{\text{\textit{invar}}}_{\F}$ (e.g. energy) and/or equivariant prediction $\Phi^{\text{\textit{equiv}}}_{\F}$ (e.g. atomic forces) by passing sequentially through a graph creation step, an embedding block, several interaction layers and an output block. $\ast$ denotes the continuous convolution operation described in Eq.~(\ref{eq:MP-eq}), $||$ denotes concatenation, figures/letters the embedding dimension and $\sigma(W \cdot)$ a 1-layer MLP with the swish activation function. Note that while the graph creation step is performed only once, we still need to sample a different canonical representation from $\mathfrak{C}$ at each epoch. For Full FA, we apply FAENet on each element of $\mathfrak{C}$ and average after the output block (not shown).}
    \label{fig:gnn-pipeline}
\end{figure*}

In this section, we describe our proposed GNN architecture, FAENet, detailed in Fig.~\ref{fig:gnn-pipeline}, which is specifically designed to take advantage of the frame averaging framework introduced in \Cref{sec:FA}. FAENet leverages 3D geometric information directly through atom relative positions without imposing any architectural constraints. FAENet inherits the theoretical guarantees of FA and satisfies relevant data symmetries. Overall, this yields greater modeling flexibility and the computational scalability needed to perform large-scale material property evaluations.

\textbf{Graph Creation}: Using $\F(D)$ from Eq.~(\ref{eq:frames}), we first apply the transformation $\rho_1$ defined in Eq.~(\ref{eq:rho-formulations}) to our input data $D$ in order to map it onto a canonical representation $\mathfrak{C}= \{\rho_1(g)^{-1}(D)| g \in \F(D)\}$ with cardinality greater than one. This data projection is computed as a pre-processing step, and has a negligible impact on training time given that PCA is calculated on a $3\times 3$ covariance matrix. During model training, following the SFA approach, we select one canonical element at random and create the graph (i.e. the adjacency matrix) using the cutoff distance $c$ based on projected positions:
\begin{align}
    A_{ij} = 
    \begin{cases}
        1 \text{ if } d_{ij} < c \\
        0 \text{ otherwise}
    \end{cases}
\end{align}
where $d_{ij} = || (\mathbf{x}_i - \mathbf{x}_j) + O_{ij*} \cdot C||$.  

\textbf{Embedding Block}: Building upon the work of~\citet{duval2022phast}, we leverage additional domain information to initialize atom representations, denoted $\h_i^{(0)} \in \R^{h}$ for atom $i$.
We define $\h_i^{(0)}$ as a concatenation of several lookup embeddings based on atomic number, group, and period, as well as a fixed set of physical properties (see Appendix~\ref{app:subsec:atom-prop}) that were found to be relevant in previous works \citep{takigawa2016machine, ward2017including}. We pass it to two-layer MLP.

To capture the 3D topology of the atomic system, we also learn an edge embedding $\mathbf{e}_{ij}$ between each pair of connected atoms $(i,j)$. Since our model architecture is free of symmetry preserving constraints, we can process geometric information in a very simple and efficient manner using a 2-layer MLP on atom relative positions $\vec{r}_{ij}$. We concatenate radial basis functions (RBF) of distance information $d_{ij}$ to $\vec{r}_{ij}$ as this quantity is very relevant and shown useful in previous works:
\begin{equation}
    \mathbf{e}_{ij} = \sigma\bigg(\text{MLP}\big(\vec{r}_{ij} || \text{RBF}(d_{ij})\big)\bigg).
    \label{eq:edge-embedding}
\end{equation}

Note that FAENet takes as input both atom relative positions and atomic number, which suffices to uniquely identify any 3D material graph\footnote{This stands in contrast to various models \citep{schutt2017schnet, klicpera2020directional, liu2021spherical} -- which are provably unable to distinguish between certain molecules (e.g.~enantiomers, CH4).}. Moreover, by processing this information using universal approximators \cite{hornik1989multilayer}, FAENet's embedding block has maximal expressive power. 

\textbf{Message Passing}: In each interaction block, we propagate messages from neighbouring atoms $j \in \mathcal{N}_i$ to the center atom $i$ using a simple continuous convolution of $\h_j^{(l)}$ (the atom embeddings at layer $l$) with 3D geometric information $\e_{ij}$. Specifically, we learn a more precise graph convolution filter $\f_{ij}^{(l)} = \sigma (\text{MLP}(\e_{ij}||\h_i^{(l)}||\h_j^{(l)}))$, with an activation function  $\sigma$, to weigh the propagated message using both geometric information and atom endpoints: 
\begin{equation}
    \mathbf{h}_i^{(l+1)}= \mathbf{h}_i^{(l)} + \text{MLP} \left( \sum_{j \in \mathcal{N}_i} \mathbf{h}_j^{(l)} \odot \f_{ij}^{(l)} \right).
    \label{eq:MP-eq}
\end{equation}

In addition to being efficient, our message passing scheme is accessible and easy-to-understand compared to many approaches that enforce symmetries architecturally. Our GNN simply weighs each message received from neighbouring nodes using a convolution filter learned via an MLP of unique 3D geometric information (i.e. relative atom positions, edge endpoints' representations). 
This novel functioning grants FAENet with great discriminative capabilities, as demonstrated empirically in \cref{app:subsec:expressivity}, where we conduct an analysis of GNNs' expressive power as proposed by \cite{joshi2022expressive}. In particular, we show that FAENet is able to distinguish between the proposed complex molecular graphs examples with perfect accuracy, better than any other method. This highlights FAENet's high expressivity and showcases the advantage of leveraging geometric information without any design restrictions. We also corroborate that SFA preserves the expressive power of the backbone architecture and does not hamper FAENet's representation capability. Please refer to \cref{app:subsec:expressivity} to see the expressivity analysis in full. 

\textbf{Architectural Details}: We use GraphNorm \cite{cai2021graphnorm}, an efficient batch normalisation method, after each message passing layer to increase our network's robustness and mitigate vanishing/exploding gradient issues. For similar purposes we add skip-connections in the message passing layer, which also help manage the over-smoothing problem \citep{chen2020measuring}. On top of that, we add jumping connections between each interaction block and the output block, using concatenation, to further boost discriminative power. Even though node representations get more refined and global as the number of layers increases, using such structural information helps the network generalize better \citep{xu2018powerful}. Finally, we chose the swish activation function because of its smooth and non-monotonic behaviour. A thorough ablation study of the model architecture is provided in \cref{app:subsec:ablation-study}.

\textbf{Output Block}: We obtain the final atom predictions ($y_i$ or $\vec{y}_i$) by passing the derived atom representations to two dense layers that map the embeddings to the correct dimension. For graph-level predictions, we perform a weighted average of atom-level predictions, where the weight is learned based on final atom embeddings: $\hat{y} = \sum_{i=1, \ldots, n} \alpha(\h^{(L)}_i) \cdot y_{i}$, where $\alpha(\cdot)$ denote the learnable importance weights. 

Note that this is the output $\Phi(D)^{\text{\textit{invar}}}$. For equivariant predictions, like forces, we need to multiply the result by $U^\top$ as indicated in Eq.~(\ref{eq:FA_approx}). Since we use an efficient approximation that samples uniformly at random one frame at each epoch among all possible ones, we directly obtain $\Phi_{\F}$.

For applications where we predict both energy and forces, unlike most previous works \cite{schutt2017schnet}, we do not compute atomic forces as the energy's gradient with respect to atom positions $F_i = \frac{\partial E}{\partial \x_i}$ (i.e. its definition in physics). The main reason for this is decreased scalability. \citet{kolluru2022open} showed that this process incurs a large computational overhead, increasing memory usage by a factor of 2-4 in addition to decreasing modeling performance in some datasets. Instead, we use two independent output heads: $\Phi^1$ for graph-level energy predictions and $\Phi^2$ for atom-level force predictions. The ability to model energy-conserving forces remains attractive in several applications such as molecular dynamics, where they are important for the stability of the simulation and improve the ability to reach new local minima \citep{chmiela2017machine}. As a result, we decide to strengthen energy-conservation by fine-tuning a new loss term: the $L_2$ norm of atomic force predictions and energy gradient with respect to atom positions:
\begin{equation}
    \mathcal{L}_{EC} = \sum_i ||\Phi^2_{\F}(D) - \nabla \Phi^1_{\F}(D)||_2.
    \label{eq:energy-conserving}
\end{equation}

\label{sec:GNN}

\section{Experiments}

\subsection{Empirical Evaluation of Model Properties}

In this section, we empirically verify the correctness of our symmetry preserving framework, including its different theoretical properties\footnote{Code is available \href{https://github.com/RolnickLab/ocp/tree/icml}{ on Github \ExternalLink}}. To do so, we consider a set of symmetry preservation metrics (formally defined in Appendix~\ref{app:sec:empirical-eval}) and evaluate them on the IS2RE and S2EF taks of OC20 \citep[see Section~\ref{sec:oc20};][]{zitnick2020introduction} for the following FAENet variants (1) SFA, (2) Full FA, (3) conventional data augmentation (DA) using randomly sampled rotations and reflections of the graph, (4) No symmetry-preservation method (No-FA), (5) No training nor FA/DA (No-Train), (6) $G=\text{SE}(3)$ instead of $\text{E}(3)$. We also compare with the invariant model SchNet and the data augmentation model ForceNet. Based on the full results provided in Appendix~\ref{app:sec:empirical-eval}, we draw the following conclusions:

\begin{itemize}
    \item \textit{Full FA} as defined in \Cref{sec:FA} yields invariant/equivariant predictions, proving the correctness of our framework and implementation.
    \item \textit{SFA} outperforms \textit{DA} and \textit{No-FA} both in terms of MAE (for comparable compute time) and of learned symmetries. In addition to providing a great approximation of invariance/equivariance, it yields slightly better MAE than \textit{Full FA} while being much faster to run, proving the overall relevance of our approach. 
    \item Not using any kind of symmetry-preserving techniques (No-FA) does not lead to invariant/equivariant predictions even if it clearly learns implicitly to preserve symmetries to some extent (vs No-Train). Besides, No-FA has a ``significant'' negative effect on performance compared to SFA, suggesting that enforcing symmetries in some way is desirable for OC20; although not critical  (possibly due to its 2D nature given fixed z-axis).
\end{itemize}

Overall, these results support the correctness of our approach and the benefits of using Stochastic FA as a symmetry preserving data augmentation method.

\subsection{Model Evaluation} \label{sec:evaluation}

In this section, we evaluate both the performance and scalability of our proposed approach, FAENet with stochastic frame averaging, on four common benchmark datasets: OC20 IS2RE, S2EF, QM9 and QM7-X. Full experimental settings are provided in the Appendix \ref{app:sec:extended-results}. 

\textbf{Metrics}: Our evaluation metrics focus on understanding the importance of the performance-scalability trade-off of different methods as motivated in Sec.~\ref{sec:intro}. Unless specified otherwise, the main modeling performance metric for each task is the  mean absolute error (MAE). Scalability is measured both by the training time for an epoch and the throughput at inference time\footnote{We define throughput as the average number of samples per second that a model can process in its forward pass.}, denoted \textit{Train} and \textit{Infer.}

\textbf{Baseline Models}: We compare our method with a wide variety of GNNs for 3D materials prediction tasks: $\text{E}(3)$-invariant models Schnet \cite{schutt2017schnet}, Dimenet++ \cite{klicpera2020fast}, GemNet-T \cite{gasteiger2021gemnet}, SphereNet \cite{liu2021spherical}, ComeNet \cite{wang2022comenet}, SpinConv \cite{shuaibi2021rotation}, GemNet-OC \cite{gasteiger2022gemnet} and Graphformer \cite{ying2021transformers}; data augmentation based ForceNet \cite{hu2021forcenet}; equivariant models So3krates \cite{frank2022so3krates}, PaiNN \cite{schutt2021equivariant} and EGNN \cite{satorras2021n}; as well as GNS \cite{godwin2021simple} and SpookyNet \cite{unke2021spookynet}.

\subsubsection{OC20}
\label{sec:oc20}

\textsl{OC20} \citep{zitnick2020introduction} is a large dataset for catalysis discovery. It was constructed to train ML models to approximate DFT for structure relaxation and energy prediction, which is fundamental to determine a catalyst’s activity and selectivity. 
\textsl{OC20} contains 1,281,040 DFT relaxations of randomly selected catalysts and adsorbates from a set of plausible candidates, where the catalyst surface is defined by a unit cell periodic in all directions. We focus on 2 tasks:
\begin{itemize}
    \item \textit{Initial Structure to Relaxed Energy} (IS2RE), that is the direct prediction of the relaxed adsorption energy from the initial atomic structure, i.e. a graph regression task requiring $\text{E}(3)$-invariance. It comes with a pre-defined train/val/test split, 450,000 training samples and hidden test labels. 
    \item \textit{Structure to Energy and Forces} (S2EF), that is the prediction of both the overall energy and atom forces, from a set of 2M 3D material structures. According to the dataset creators, the 2M split closely approximates the much more expensive full S2EF dataset, making it suitable for model evaluation \cite{chanussot2021open}.
\end{itemize}

We evaluated all models on the four $\sim$25K samples splits of the validation set:  In Domain (ID), Out of Domain Adsorbates (OOD-ads), Out of Domain catalysts (OOD-cat), and Out of Domain Adsorbates and catalysts (OOD-both). We measure performance on each validation split by the energy MAE (IS2RE, S2EF), atomic forces MAE (S2EF) and the percentage of predicted Energies within a Threshold (EwT) of the ground truth energy (IS2RE). Running times are taken over the ID validation set on similar GPU types\footnote{If anything, the A6000 used by ComENet is benchmarked to be faster than the RTX8000 we have used for these models.}. 

\begin{figure}[h]
    \centering
    \includegraphics[width=0.43\textwidth]{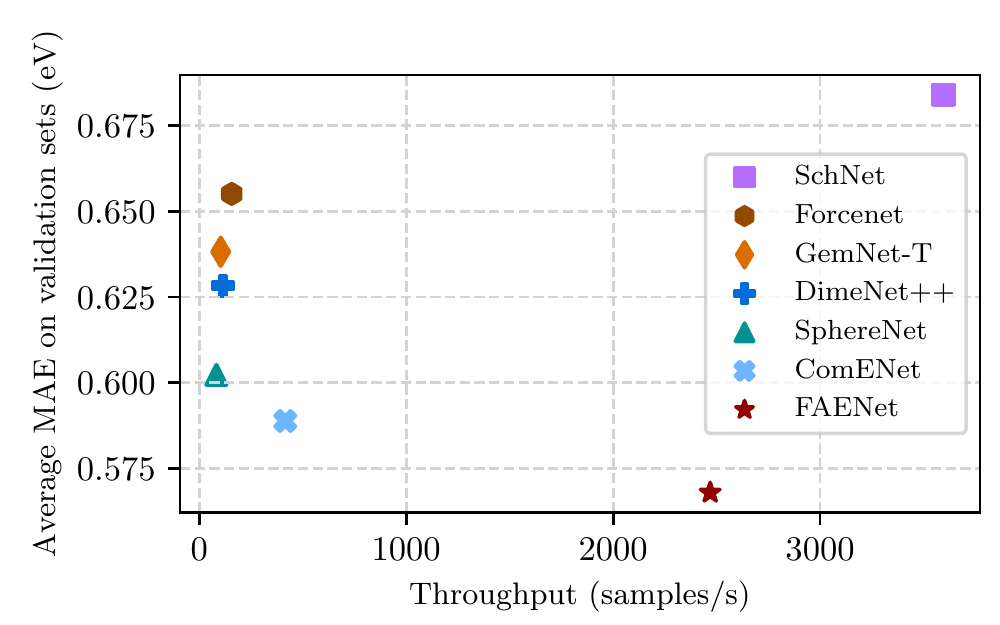}
    \caption{OC-20 Direct IS2RE performance / scalability trade-off. FAENet achieves the best MAE across methods while being much faster than the majority of baselines.}
    \label{fig:is2re_pareto}
\end{figure}

\textbf{Results IS2RE}: ~\Cref{fig:is2re_pareto} and the corresponding full table of results (\Cref{tb:result_oc20_is2re_full}) show that FAENet outperforms all existing baseline approaches in terms of Energy MAE and EwT, except for Graphormer \citep{ying2021transformers}. While we could not obtain comparable running times for Graphormer, we note that it uses an ensemble of 31 models and requires 372 GPU-days to train with A100 GPUs. Similarly, GNS, the runner up of FAENet in terms of MAE, is a 50-layer model which is extremely expensive to run. Lastly, FAENet is 21 times faster than DimeNet++ and 6 times faster than ComeNet at inference time (\textit{w.r.t.} throughput). On the other hand, while SchNet is still slightly faster than FAENet, we improve the average 
validation MAE by 17\%. In a word, as illustrated in the Pareto plot in \Cref{fig:is2re_pareto}, FAENet presents an extremely attractive performance-scalability trade-off.

\textbf{Results S2EF 2M}: Given the substantial cost of training on S2EF compared to IS2RE, results reported in \Cref{tb:result_oc20_s2ef} include fewer but still enough baselines to showcase the great performance-scalability feature of FAENet. Indeed, FAENet outperforms most baselines in terms of Energy MAE and Force MAE, while presenting the shortest runtime by a sizeable margin (e.g. even faster than SchNet on this dataset for optimal model architectures). Compared to ForceNet, the only other data augmentation approach, FAENet improves energy MAE by $38\%$ while decreasing training time for an epoch by a factor of 20. This illustrates the benefits of our symmetry-preserving DA framework combined with our efficient message passing GNN design. Similarly to Graphormer in IS2RE, GemNet-OC achieves outstanding results on this dataset. Yet, GemNet-OC is significantly larger and more computationally expensive, processing 34 times less samples by GPU seconds than FAENet with respect to inference throughput \citep{gasteiger2022gemnet}. GemNet-OC's functioning is also more complex given the constraints imposed on its functioning. Once again, FAENet offers an extremely desirable performance-scalability compromise.

\begin{table}[h]
\centering
        \caption{Results on OC20 S2EF 2M, averaged accross all 4 validation splits. Performance is measured in terms of energy and forces MAEs. The best one is bolded, the second best is underlined. Scalability is measured with training time for one epoch (\textit{Train}, in minutes) and inference throughput (\textit{Infer.}, samples per second at inference time). Table \ref{tb:result_oc20_s2ef_full} shows each val. split score. 
        }
    \label{tb:result_oc20_s2ef}
    \begin{tabular}{lcccc}
    \toprule
    & \multicolumn{2}{c}{Scalability} &\multicolumn{2}{c}{Average MAE} \\
Model     & Train $\downarrow$ & Infer. $\uparrow$ & Energy $\downarrow$ & Forces $\uparrow$\\
\midrule
SchNet    & \underline{98min}              & \underline{509}               & 1.4917              & 83.1 \\
DimeNet++ & 1,157min            & 52                & 0.8096              & 66.3 \\
ForceNet  & 1,476min                 & 36                & 0.7548                  & 61.0 \\
GemNet-OC & --                 & 18                & \textbf{0.2860}     & \textbf{25.7} \\
FAENet    & \textbf{75min}              & \textbf{623}               & \underline{0.4642}  & \underline{57.5} \\
\bottomrule
\end{tabular}
\end{table}

\subsubsection{QM7-X}

\textbf{QM7-X} \citep{hoja2021qm7} is a dataset containing ~7K molecular graphs with up to seven non-hydrogen atoms (C, N, O, S, Cl), drawn from the GDB13 chemical universe \citep{blum2009970}. After sampling and optimizing the structural and constitutional (stereo)isomers of each graph, the authors obtained 42K equilibrium structures. Each of them was then perturbed in order to obtain 99 additional non-equilibrium structures, leading to a total of about 4.2M samples. Of the 47 available targets, we focus on energy and forces as per previous work~\cite{unke2021spookynet,frank2022so3krates} but also because they are the main quantities of interest in real-world applications we care about.

\textbf{Results}: We report FAENet's results in Table~\ref{table:app:qm7x} alongside the performance of SchNet, PaiNN, So3krates and SpookyNet. Since QM7-X is a very recent benchmark, only SpookyNet (which also computes SchNet and PaiNN performance) and So3krates have reported results on it for a comparable task. Unfortunately, the authors of the SpookyNet did not provide their code implementations, making it impractical for us confirm the results. We were able to reproduce the results for SchNet after authors of So3krates shared their codes with us, which included an unreported normalization method to train on shifted energy targets\footnote{\href{https://github.com/thorben-frank/mlff/blob/v0.1/mlff/examples/02_Multiple_Structure_Training.ipynb}{https://github.com/thorben-frank/mlff}}. Nonetheless, we observe that FAENet achieves close-to-SOTA performance in spite of its simple architecture.

\begin{table}[h]
\centering
\caption{QM7x results for FAENet and baseline methods based on best-effort replication given limited code availability. FAENet achieves competitive MAE. Best in bold, second best underlined.}
\label{table:app:qm7x}
\begin{tabular}{lcccc}
\hline
\multicolumn{1}{c}{\textbf{}} & \multicolumn{2}{c}{\textit{Known molecules}} & \multicolumn{2}{c}{\textit{Unknown molecules}} \\
Model     & Energy & Forces & Energy & Forces \\ \hline
SchNet    & 42.43 & 56.45 & 51.05  & 65.84 \\
PaiNN     & 15.69 & 20.30 & 17.59 & 24.16 \\
So3krates & 15.22 & 18.44 & 21.75 & 23.16 \\
SpookyNet & \textbf{10.62} & \textbf{14.85}  & \textbf{13.15} & \textbf{17.32}  \\
FAENet    & \underline{11.42}  & \underline{17.54} & \underline{15.17}  & \underline{21.23}      \\ \hline
\end{tabular}
\end{table}

\subsubsection{QM9}

\textbf{QM9} \cite{ramakrishnan2014quantum} is a widely used dataset for molecular property prediction. It includes geometric, energetic, electronic, and thermodynamic properties for 134K stable small organic molecules with up to 9 heavy C, O, N, and F atoms. The dataset is split into 110K molecules for training, 10K for validation, and remaining 14K for testing.

\textbf{QM9 Evaluation}: We compare FAENet to a variety of methods that have publicly shared implementations and results. We trained independently on each target property a single FAENet model. Following \cite{liu2021spherical}, we report each model's MAE and training time for a single epoch. We also report its mean error and associated standard deviation across all molecular properties in QM9. To enable comparisons across properties units and ranges, we compute the average \textit{relative improvement}\footnote{For a model's MAE $m_p$ on property $p$, we report $r = \frac{m_p^0 - m_p}{m_p^0}$ where $m_p^0$ is the MAE of SchNet on $p$.} with respect to SchNet.

\textbf{Results}: The full table of results can be found in the Appendix~\ref{tb:qm9-results}. We plot relative improvements against the time we have measured models to take over a training epoch in Fig.~\ref{fig:qm9-pareto}. In addition to achieving competitive modeling performance, FAENet is the only model faster than SchNet. 
\begin{figure}[h]
    \centering
    \includegraphics[width=0.48\textwidth]{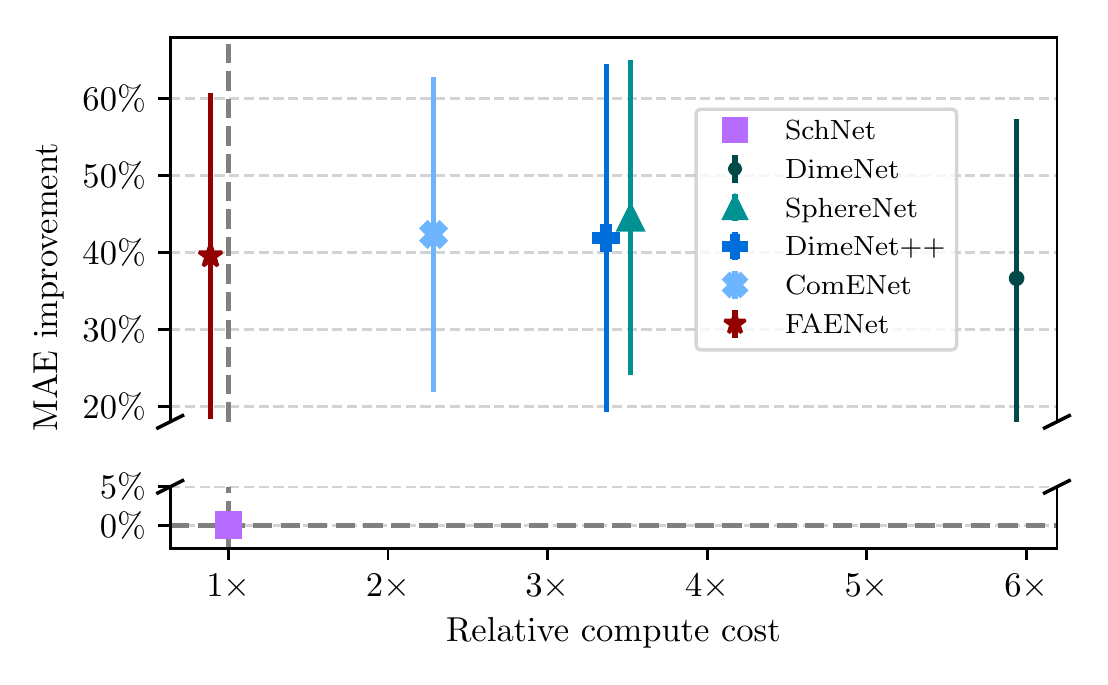}
    \vspace{-1cm}
    \caption{Performance/scalability trade-off on QM9. Large standard deviations and close means on the MAE average relative improvement against SchNet show that in recent years, models have saturated the complexity of QM9. There is however a wide range of (measured) computational costs. FAENet is the fastest.}
    \label{fig:qm9-pareto}
\end{figure}


\label{sec:Experiments}

\section{Conclusion}
In this paper, we offer an alternative to traditional symmetry-preserving GNN models for 3D materials property prediction. Instead of constraining the GNN functioning, we propose a flexible model-agnostic framework to enforce symmetries through a carefully defined data augmentation method built upon frame averaging. We then use this framework to design FAENet: a simple, fast and expressive GNN model, which directly and unrestrictedly processes atom relative information. While our empirical analysis demonstrates that Stochastic FA provides near-perfect symmetry preservation at up to 8$\times$ faster inference time, the rest of our experiments show that FAENet (with SFA) achieves competitive modeling performance and superior scalability on various materials modeling tasks.

\textbf{Future work}. While our approach provides clear advantages, future work can improve upon current limitations by studying the conditions when exact equivariance can be imposed with less than the maximum number of frames. This could be particularly important in high-risk applications where guarantees are required to ensure safety. Additionally, both modeling performance and compute scalability require further improvement to effectively enable large-scale computational design of novel materials systems. Once achieved, the combination of fast and robust property prediction GNN with ML-based generative methods like diffusion models, reinforcement learning or GFlowNets, shall enable to meaningfully explore the vast search space in materials discovery.

\textbf{Acknowledgements}. This research is supported in part by ANR (French National Research Agency) under the JCJC project GraphIA (ANR-20-CE23-0009-01). Alexandre Duval acknowledges support from a Mitacs Globalink Research Award. Alex Hernandez-Garcia acknowledges the support of IVADO and the Canada First Research Excellence Fund. David Rolnick acknowledges support from the Canada CIFAR AI Chairs Program. The authors also acknowledge material support from NVIDIA and Intel in the form of computational resources, and are grateful for technical support from the Mila IDT team in maintaining the Mila Compute Cluster.

\bibliography{main}

\begin{thebibliography}{58}
\providecommand{\natexlab}[1]{#1}
\providecommand{\url}[1]{\texttt{#1}}
\expandafter\ifx\csname urlstyle\endcsname\relax
  \providecommand{\doi}[1]{doi: #1}\else
  \providecommand{\doi}{doi: \begingroup \urlstyle{rm}\Url}\fi

\bibitem[Adams et~al.(2021)Adams, Pattanaik, and Coley]{adams2021learning}
Adams, K., Pattanaik, L., and Coley, C.~W.
\newblock Learning 3d representations of molecular chirality with invariance to
  bond rotations.
\newblock \emph{Preprint arXiv:2110.04383}, 2021.

\bibitem[Agrawal \& Choudhary(2016)Agrawal and
  Choudhary]{agrawal2016fourthparadigm}
Agrawal, A. and Choudhary, A.
\newblock Perspective: Materials informatics and big data: Realization of the
  “fourth paradigm” of science in materials science.
\newblock \emph{Apl Materials}, 4\penalty0 (5):\penalty0 053208, 2016.

\bibitem[Anderson et~al.(2019)Anderson, Hy, and Kondor]{anderson2019cormorant}
Anderson, B., Hy, T.~S., and Kondor, R.
\newblock Cormorant: Covariant molecular neural networks.
\newblock \emph{Advances in neural information processing systems}, 32, 2019.

\bibitem[Atz et~al.(2021)Atz, Grisoni, and Schneider]{atz2021geometric}
Atz, K., Grisoni, F., and Schneider, G.
\newblock Geometric deep learning on molecular representations.
\newblock \emph{Nature Machine Intelligence}, 3\penalty0 (12):\penalty0
  1023--1032, 2021.

\bibitem[Batatia et~al.(2022)Batatia, Kov{\'a}cs, Simm, Ortner, and
  Cs{\'a}nyi]{batatia2022mace}
Batatia, I., Kov{\'a}cs, D.~P., Simm, G.~N., Ortner, C., and Cs{\'a}nyi, G.
\newblock Mace: Higher order equivariant message passing neural networks for
  fast and accurate force fields.
\newblock \emph{Preprint arXiv:2206.07697}, 2022.

\bibitem[Batzner et~al.(2022)Batzner, Musaelian, Sun, Geiger, Mailoa,
  Kornbluth, Molinari, Smidt, and Kozinsky]{batzner20223}
Batzner, S., Musaelian, A., Sun, L., Geiger, M., Mailoa, J.~P., Kornbluth, M.,
  Molinari, N., Smidt, T.~E., and Kozinsky, B.
\newblock E(3)-equivariant graph neural networks for data-efficient and
  accurate interatomic potentials.
\newblock \emph{Nature communications}, 13\penalty0 (1):\penalty0 1--11, 2022.

\bibitem[Blum \& Reymond(2009)Blum and Reymond]{blum2009970}
Blum, L.~C. and Reymond, J.-L.
\newblock 970 million druglike small molecules for virtual screening in the
  chemical universe database {GDB}-13.
\newblock \emph{Journal of the American Chemical Society}, 131\penalty0
  (25):\penalty0 8732--8733, 2009.

\bibitem[Bohacek et~al.(1996)Bohacek, McMartin, and
  Guida]{bohacek1996molecularmodeling}
Bohacek, R.~S., McMartin, C., and Guida, W.~C.
\newblock The art and practice of structure-based drug design: A molecular
  modeling perspective.
\newblock \emph{Medicinal Research Reviews}, 16\penalty0 (1):\penalty0 3--50,
  1996.

\bibitem[Brandstetter et~al.(2021)Brandstetter, Hesselink, van~der Pol,
  Bekkers, and Welling]{brandstetter2021geometric}
Brandstetter, J., Hesselink, R., van~der Pol, E., Bekkers, E., and Welling, M.
\newblock Geometric and physical quantities improve {E}(3) equivariant message
  passing.
\newblock \emph{Preprint arXiv:2110.02905}, 2021.

\bibitem[Cai et~al.(2021)Cai, Luo, Xu, He, Liu, and Wang]{cai2021graphnorm}
Cai, T., Luo, S., Xu, K., He, D., Liu, T.-y., and Wang, L.
\newblock Graphnorm: A principled approach to accelerating graph neural network
  training.
\newblock In \emph{International Conference on Machine Learning}, pp.\
  1204--1215. PMLR, 2021.

\bibitem[Chanussot et~al.(2021)Chanussot, Das, Goyal, Lavril, Shuaibi, Riviere,
  Tran, Heras-Domingo, Ho, Hu, et~al.]{chanussot2021open}
Chanussot, L., Das, A., Goyal, S., Lavril, T., Shuaibi, M., Riviere, M., Tran,
  K., Heras-Domingo, J., Ho, C., Hu, W., et~al.
\newblock Open catalyst 2020 (oc20) dataset and community challenges.
\newblock \emph{ACS Catalysis}, 11\penalty0 (10):\penalty0 6059--6072, 2021.

\bibitem[Chen \& Ong(2022)Chen and Ong]{chen2022universal}
Chen, C. and Ong, S.~P.
\newblock A universal graph deep learning interatomic potential for the
  periodic table.
\newblock \emph{Preprint arXiv:2202.02450}, 2022.

\bibitem[Chen et~al.(2020)Chen, Lin, Li, Li, Zhou, and Sun]{chen2020measuring}
Chen, D., Lin, Y., Li, W., Li, P., Zhou, J., and Sun, X.
\newblock Measuring and relieving the over-smoothing problem for graph neural
  networks from the topological view.
\newblock In \emph{Proceedings of the AAAI Conference on Artificial
  Intelligence}, volume~34, pp.\  3438--3445, 2020.

\bibitem[Chen et~al.(2021)Chen, Liu, Chen, Li, and Hill]{chen2021equivariant}
Chen, H., Liu, S., Chen, W., Li, H., and Hill, R.
\newblock Equivariant point network for {3D} point cloud analysis.
\newblock In \emph{Proceedings of the IEEE/CVF Conference on Computer Vision
  and Pattern Recognition}, pp.\  14514--14523, 2021.

\bibitem[Chmiela et~al.(2017)Chmiela, Tkatchenko, Sauceda, Poltavsky,
  Sch{\"u}tt, and M{\"u}ller]{chmiela2017machine}
Chmiela, S., Tkatchenko, A., Sauceda, H.~E., Poltavsky, I., Sch{\"u}tt, K.~T.,
  and M{\"u}ller, K.-R.
\newblock Machine learning of accurate energy-conserving molecular force
  fields.
\newblock \emph{Science advances}, 3\penalty0 (5):\penalty0 e1603015, 2017.

\bibitem[Duval et~al.(2022)Duval, Schmidt, Miret, Bengio,
  Hern{\'a}ndez-Garc{\'\i}a, and Rolnick]{duval2022phast}
Duval, A., Schmidt, V., Miret, S., Bengio, Y., Hern{\'a}ndez-Garc{\'\i}a, A.,
  and Rolnick, D.
\newblock {PhAST}: Physics-aware, scalable, and task-specific {GNNs} for
  accelerated catalyst design.
\newblock \emph{Preprint arXiv:2211.12020}, 2022.

\bibitem[Fey \& Lenssen(2019)Fey and Lenssen]{Fey/Lenssen/2019}
Fey, M. and Lenssen, J.~E.
\newblock Fast graph representation learning with {PyTorch Geometric}.
\newblock In \emph{ICLR Workshop on Representation Learning on Graphs and
  Manifolds}, 2019.

\bibitem[Frank et~al.(2022)Frank, Unke, and Muller]{frank2022so3krates}
Frank, T., Unke, O.~T., and Muller, K.~R.
\newblock So3krates: Equivariant attention for interactions on arbitrary
  length-scales in molecular systems.
\newblock In \emph{Advances in Neural Information Processing Systems}, 2022.

\bibitem[Fuchs et~al.(2020)Fuchs, Worrall, Fischer, and Welling]{fuchs2020se}
Fuchs, F., Worrall, D., Fischer, V., and Welling, M.
\newblock {SE}(3)-transformers: {3D} roto-translation equivariant attention
  networks.
\newblock \emph{Advances in Neural Information Processing Systems},
  33:\penalty0 1970--1981, 2020.

\bibitem[Gasteiger et~al.(2021)Gasteiger, Becker, and
  G{\"u}nnemann]{gasteiger2021gemnet}
Gasteiger, J., Becker, F., and G{\"u}nnemann, S.
\newblock Gemnet: Universal directional graph neural networks for molecules.
\newblock \emph{Advances in Neural Information Processing Systems},
  34:\penalty0 6790--6802, 2021.

\bibitem[Gasteiger et~al.(2022)Gasteiger, Shuaibi, Sriram, G{\"u}nnemann,
  Ulissi, Zitnick, and Das]{gasteiger2022gemnet}
Gasteiger, J., Shuaibi, M., Sriram, A., G{\"u}nnemann, S., Ulissi, Z., Zitnick,
  C.~L., and Das, A.
\newblock Gemnet-oc: developing graph neural networks for large and diverse
  molecular simulation datasets.
\newblock \emph{arXiv preprint arXiv:2204.02782}, 2022.

\bibitem[Gerken et~al.(2022)Gerken, Carlsson, Linander, Ohlsson, Petersson, and
  Persson]{gerken2022equivariance}
Gerken, J., Carlsson, O., Linander, H., Ohlsson, F., Petersson, C., and
  Persson, D.
\newblock Equivariance versus augmentation for spherical images.
\newblock In \emph{International Conference on Machine Learning}, pp.\
  7404--7421. PMLR, 2022.

\bibitem[Godwin et~al.(2021)Godwin, Schaarschmidt, Gaunt, Sanchez-Gonzalez,
  Rubanova, Veli{\v{c}}kovi{\'c}, Kirkpatrick, and Battaglia]{godwin2021simple}
Godwin, J., Schaarschmidt, M., Gaunt, A., Sanchez-Gonzalez, A., Rubanova, Y.,
  Veli{\v{c}}kovi{\'c}, P., Kirkpatrick, J., and Battaglia, P.
\newblock Simple gnn regularisation for 3d molecular property prediction \&
  beyond.
\newblock \emph{arXiv preprint arXiv:2106.07971}, 2021.

\bibitem[Han et~al.(2022)Han, Rong, Xu, and Huang]{han2022geometrically}
Han, J., Rong, Y., Xu, T., and Huang, W.
\newblock Geometrically equivariant graph neural networks: A survey.
\newblock \emph{Preprint arXiv:2202.07230}, 2022.

\bibitem[Harris et~al.(2020)Harris, Millman, van~der Walt, Gommers, Virtanen,
  Cournapeau, Wieser, Taylor, Berg, Smith, Kern, Picus, Hoyer, van Kerkwijk,
  Brett, Haldane, del R{\'{i}}o, Wiebe, Peterson, G{\'{e}}rard-Marchant,
  Sheppard, Reddy, Weckesser, Abbasi, Gohlke, and Oliphant]{harris2020array}
Harris, C.~R., Millman, K.~J., van~der Walt, S.~J., Gommers, R., Virtanen, P.,
  Cournapeau, D., Wieser, E., Taylor, J., Berg, S., Smith, N.~J., Kern, R.,
  Picus, M., Hoyer, S., van Kerkwijk, M.~H., Brett, M., Haldane, A., del
  R{\'{i}}o, J.~F., Wiebe, M., Peterson, P., G{\'{e}}rard-Marchant, P.,
  Sheppard, K., Reddy, T., Weckesser, W., Abbasi, H., Gohlke, C., and Oliphant,
  T.~E.
\newblock Array programming with {NumPy}.
\newblock \emph{Nature}, 585\penalty0 (7825):\penalty0 357--362, September
  2020.
\newblock \doi{10.1038/s41586-020-2649-2}.

\bibitem[Hoja et~al.(2021)Hoja, Medrano~Sandonas, Ernst, Vazquez-Mayagoitia,
  DiStasio~Jr, and Tkatchenko]{hoja2021qm7}
Hoja, J., Medrano~Sandonas, L., Ernst, B.~G., Vazquez-Mayagoitia, A.,
  DiStasio~Jr, R.~A., and Tkatchenko, A.
\newblock Qm7-x, a comprehensive dataset of quantum-mechanical properties
  spanning the chemical space of small organic molecules.
\newblock \emph{Scientific data}, 8\penalty0 (1):\penalty0 1--11, 2021.

\bibitem[Hornik et~al.(1989)Hornik, Stinchcombe, and
  White]{hornik1989multilayer}
Hornik, K., Stinchcombe, M., and White, H.
\newblock Multilayer feedforward networks are universal approximators.
\newblock \emph{Neural networks}, 2\penalty0 (5):\penalty0 359--366, 1989.

\bibitem[Hu et~al.(2021)Hu, Shuaibi, Das, Goyal, Sriram, Leskovec, Parikh, and
  Zitnick]{hu2021forcenet}
Hu, W., Shuaibi, M., Das, A., Goyal, S., Sriram, A., Leskovec, J., Parikh, D.,
  and Zitnick, C.~L.
\newblock Forcenet: A graph neural network for large-scale quantum
  calculations.
\newblock \emph{Preprint arXiv:2103.01436}, 2021.

\bibitem[Joshi et~al.(2022)Joshi, Bodnar, Mathis, Cohen, and
  Li{\`o}]{joshi2022expressive}
Joshi, C.~K., Bodnar, C., Mathis, S.~V., Cohen, T., and Li{\`o}, P.
\newblock On the expressive power of geometric graph neural networks.
\newblock In \emph{The First Learning on Graphs Conference}, 2022.

\bibitem[Klicpera et~al.(2020{\natexlab{a}})Klicpera, Giri, Margraf, and
  G{\"u}nnemann]{klicpera2020fast}
Klicpera, J., Giri, S., Margraf, J.~T., and G{\"u}nnemann, S.
\newblock Fast and uncertainty-aware directional message passing for
  non-equilibrium molecules.
\newblock \emph{Preprint arXiv:2011.14115}, 2020{\natexlab{a}}.

\bibitem[Klicpera et~al.(2020{\natexlab{b}})Klicpera, Gro{\ss}, and
  G{\"u}nnemann]{klicpera2020directional}
Klicpera, J., Gro{\ss}, J., and G{\"u}nnemann, S.
\newblock Directional message passing for molecular graphs.
\newblock \emph{Preprint arXiv:2003.03123}, 2020{\natexlab{b}}.

\bibitem[Kolluru et~al.(2022)Kolluru, Shuaibi, Palizhati, Shoghi, Das, Wood,
  Zitnick, Kitchin, and Ulissi]{kolluru2022open}
Kolluru, A., Shuaibi, M., Palizhati, A., Shoghi, N., Das, A., Wood, B.,
  Zitnick, C.~L., Kitchin, J.~R., and Ulissi, Z.~W.
\newblock Open challenges in developing generalizable large scale machine
  learning models for catalyst discovery.
\newblock \emph{Preprint arXiv:2206.02005}, 2022.

\bibitem[Liu et~al.(2021)Liu, Wang, Liu, Zhang, Oztekin, and
  Ji]{liu2021spherical}
Liu, Y., Wang, L., Liu, M., Zhang, X., Oztekin, B., and Ji, S.
\newblock Spherical message passing for {3D} graph networks.
\newblock \emph{Preprint arXiv:2102.05013}, 2021.

\bibitem[Mentel()]{mendeleev2014}
Mentel, L.
\newblock {mendeleev} -- a python resource for properties of chemical elements,
  ions and isotopes.
\newblock URL \url{https://github.com/lmmentel/mendeleev}.

\bibitem[Miret et~al.(2022)Miret, Lee, Gonzales, Nassar, and
  Spellings]{miret2022open}
Miret, S., Lee, K. L.~K., Gonzales, C., Nassar, M., and Spellings, M.
\newblock The open {MatSci} {ML} toolkit: A flexible framework for machine
  learning in materials science.
\newblock \emph{Preprint arXiv:2210.17484}, 2022.

\bibitem[Morris et~al.(2019)Morris, Ritzert, Fey, Hamilton, Lenssen, Rattan,
  and Grohe]{morris2019weisfeiler}
Morris, C., Ritzert, M., Fey, M., Hamilton, W.~L., Lenssen, J.~E., Rattan, G.,
  and Grohe, M.
\newblock Weisfeiler and leman go neural: Higher-order graph neural networks.
\newblock In \emph{Proceedings of the AAAI conference on artificial
  intelligence}, volume~33, pp.\  4602--4609, 2019.

\bibitem[Paszke et~al.(2019)Paszke, Gross, Massa, Lerer, Bradbury, Chanan,
  Killeen, Lin, Gimelshein, Antiga, Desmaison, Kopf, Yang, DeVito, Raison,
  Tejani, Chilamkurthy, Steiner, Fang, Bai, and Chintala]{NEURIPS2019_bdbca288}
Paszke, A., Gross, S., Massa, F., Lerer, A., Bradbury, J., Chanan, G., Killeen,
  T., Lin, Z., Gimelshein, N., Antiga, L., Desmaison, A., Kopf, A., Yang, E.,
  DeVito, Z., Raison, M., Tejani, A., Chilamkurthy, S., Steiner, B., Fang, L.,
  Bai, J., and Chintala, S.
\newblock Pytorch: An imperative style, high-performance deep learning library.
\newblock In Wallach, H., Larochelle, H., Beygelzimer, A., d\textquotesingle
  Alch\'{e}-Buc, F., Fox, E., and Garnett, R. (eds.), \emph{Advances in Neural
  Information Processing Systems}, volume~32. Curran Associates, Inc., 2019.

\bibitem[Pozdnyakov \& Ceriotti(2022)Pozdnyakov and
  Ceriotti]{pozdnyakov2022incompleteness}
Pozdnyakov, S.~N. and Ceriotti, M.
\newblock Incompleteness of graph convolutional neural networks for points
  clouds in three dimensions.
\newblock \emph{arXiv preprint}, 2022.

\bibitem[Puny et~al.(2022)Puny, Atzmon, Smith, Misra, Grover, Ben-Hamu, and
  Lipman]{puny2022frame}
Puny, O., Atzmon, M., Smith, E.~J., Misra, I., Grover, A., Ben-Hamu, H., and
  Lipman, Y.
\newblock Frame averaging for invariant and equivariant network design.
\newblock In \emph{International Conference on Learning Representations}, 2022.

\bibitem[Ramakrishnan et~al.(2014)Ramakrishnan, Dral, Rupp, and
  Von~Lilienfeld]{ramakrishnan2014quantum}
Ramakrishnan, R., Dral, P.~O., Rupp, M., and Von~Lilienfeld, O.~A.
\newblock Quantum chemistry structures and properties of 134 kilo molecules.
\newblock \emph{Scientific data}, 1\penalty0 (1):\penalty0 1--7, 2014.

\bibitem[Satorras et~al.(2021)Satorras, Hoogeboom, and Welling]{satorras2021n}
Satorras, V.~G., Hoogeboom, E., and Welling, M.
\newblock E(n) equivariant graph neural networks.
\newblock In \emph{International conference on machine learning}, pp.\
  9323--9332. PMLR, 2021.

\bibitem[Sch{\"u}tt et~al.(2017)Sch{\"u}tt, Kindermans, Sauceda~Felix, Chmiela,
  Tkatchenko, and M{\"u}ller]{schutt2017schnet}
Sch{\"u}tt, K., Kindermans, P.-J., Sauceda~Felix, H.~E., Chmiela, S.,
  Tkatchenko, A., and M{\"u}ller, K.-R.
\newblock Schnet: A continuous-filter convolutional neural network for modeling
  quantum interactions.
\newblock \emph{Advances in neural information processing systems}, 30, 2017.

\bibitem[Sch{\"u}tt et~al.(2021)Sch{\"u}tt, Unke, and
  Gastegger]{schutt2021equivariant}
Sch{\"u}tt, K., Unke, O., and Gastegger, M.
\newblock Equivariant message passing for the prediction of tensorial
  properties and molecular spectra.
\newblock In \emph{International Conference on Machine Learning}, pp.\
  9377--9388. PMLR, 2021.

\bibitem[Shuaibi et~al.(2021)Shuaibi, Kolluru, Das, Grover, Sriram, Ulissi, and
  Zitnick]{shuaibi2021rotation}
Shuaibi, M., Kolluru, A., Das, A., Grover, A., Sriram, A., Ulissi, Z., and
  Zitnick, C.~L.
\newblock Rotation invariant graph neural networks using spin convolutions.
\newblock \emph{Preprint arXiv:2106.09575}, 2021.

\bibitem[Smidt(2021)]{smidt2021euclidean}
Smidt, T.~E.
\newblock Euclidean symmetry and equivariance in machine learning.
\newblock \emph{Trends in Chemistry}, 3\penalty0 (2):\penalty0 82--85, 2021.

\bibitem[Takigawa et~al.(2016)Takigawa, Shimizu, Tsuda, and
  Takakusagi]{takigawa2016machine}
Takigawa, I., Shimizu, K.-i., Tsuda, K., and Takakusagi, S.
\newblock Machine-learning prediction of the d-band center for metals and
  bimetals.
\newblock \emph{RSC advances}, 6\penalty0 (58):\penalty0 52587--52595, 2016.

\bibitem[Th{\"o}lke \& De~Fabritiis(2022)Th{\"o}lke and
  De~Fabritiis]{tholke2022torchmd}
Th{\"o}lke, P. and De~Fabritiis, G.
\newblock Torchmd-net: Equivariant transformers for neural network based
  molecular potentials.
\newblock \emph{Preprint arXiv:2202.02541}, 2022.

\bibitem[Thomas et~al.(2018)Thomas, Smidt, Kearnes, Yang, Li, Kohlhoff, and
  Riley]{thomas2018tensor}
Thomas, N., Smidt, T., Kearnes, S., Yang, L., Li, L., Kohlhoff, K., and Riley,
  P.
\newblock Tensor field networks: Rotation-and translation-equivariant neural
  networks for 3d point clouds.
\newblock \emph{Preprint arXiv:1802.08219}, 2018.

\bibitem[Unke \& Meuwly(2019)Unke and Meuwly]{unke2019physnet}
Unke, O.~T. and Meuwly, M.
\newblock {PhysNet}: A neural network for predicting energies, forces, dipole
  moments, and partial charges.
\newblock \emph{Journal of chemical theory and computation}, 15\penalty0
  (6):\penalty0 3678--3693, 2019.

\bibitem[Unke et~al.(2021)Unke, Chmiela, Gastegger, Sch{\"u}tt, Sauceda, and
  M{\"u}ller]{unke2021spookynet}
Unke, O.~T., Chmiela, S., Gastegger, M., Sch{\"u}tt, K.~T., Sauceda, H.~E., and
  M{\"u}ller, K.-R.
\newblock Spookynet: Learning force fields with electronic degrees of freedom
  and nonlocal effects.
\newblock \emph{Nature communications}, 12\penalty0 (1):\penalty0 1--14, 2021.

\bibitem[Wang et~al.(2022)Wang, Liu, Lin, Liu, and Ji]{wang2022comenet}
Wang, L., Liu, Y., Lin, Y., Liu, H., and Ji, S.
\newblock {ComENet}: Towards complete and efficient message passing for {3D}
  molecular graphs.
\newblock \emph{Preprint arXiv:2206.08515}, 2022.

\bibitem[Wang \& Zhang(2022)Wang and Zhang]{wang2022graph}
Wang, X. and Zhang, M.
\newblock Graph neural network with local frame for molecular potential energy
  surface.
\newblock \emph{Preprint arXiv:2208.00716}, 2022.

\bibitem[Wang et~al.(2020)Wang, Ren, Yan, Guo, Zhang, and Wonka]{wang2020mgcn}
Wang, Y., Ren, J., Yan, D.-M., Guo, J., Zhang, X., and Wonka, P.
\newblock Mgcn: Descriptor learning using multiscale gcns.
\newblock \emph{ACM Transactions on Graphics (TOG)}, 39\penalty0 (4):\penalty0
  122--1, 2020.

\bibitem[Ward et~al.(2017)Ward, Liu, Krishna, Hegde, Agrawal, Choudhary, and
  Wolverton]{ward2017including}
Ward, L., Liu, R., Krishna, A., Hegde, V.~I., Agrawal, A., Choudhary, A., and
  Wolverton, C.
\newblock Including crystal structure attributes in machine learning models of
  formation energies via voronoi tessellations.
\newblock \emph{Physical Review B}, 96\penalty0 (2):\penalty0 024104, 2017.

\bibitem[Xu et~al.(2018)Xu, Hu, Leskovec, and Jegelka]{xu2018powerful}
Xu, K., Hu, W., Leskovec, J., and Jegelka, S.
\newblock How powerful are graph neural networks?
\newblock \emph{Preprint arXiv:1810.00826}, 2018.

\bibitem[Yarotsky(2022)]{yarotsky2022universal}
Yarotsky, D.
\newblock Universal approximations of invariant maps by neural networks.
\newblock \emph{Constructive Approximation}, 55\penalty0 (1):\penalty0
  407--474, 2022.

\bibitem[Ying et~al.(2021)Ying, Cai, Luo, Zheng, Ke, He, Shen, and
  Liu]{ying2021transformers}
Ying, C., Cai, T., Luo, S., Zheng, S., Ke, G., He, D., Shen, Y., and Liu, T.-Y.
\newblock Do transformers really perform badly for graph representation?
\newblock \emph{Advances in Neural Information Processing Systems},
  34:\penalty0 28877--28888, 2021.

\bibitem[Zitnick et~al.(2020)Zitnick, Chanussot, Das, Goyal, Heras-Domingo, Ho,
  Hu, Lavril, Palizhati, Riviere, et~al.]{zitnick2020introduction}
Zitnick, C.~L., Chanussot, L., Das, A., Goyal, S., Heras-Domingo, J., Ho, C.,
  Hu, W., Lavril, T., Palizhati, A., Riviere, M., et~al.
\newblock An introduction to electrocatalyst design using machine learning for
  renewable energy storage.
\newblock \emph{Preprint arXiv:2010.09435}, 2020.

\end{thebibliography}
\bibliographystyle{icml2023}

\newpage
\appendix
\onecolumn

\section{Frame averaging}
\label{app:sec:SFA}

\begin{proof}
Let $D=(X,Z,C,O) \in V$ be an arbitrary atomic system with $X, Z, C, O$ as given in Sec.~\ref{subsec:frame-construction}. As our definition of the frame $\F$ in Eq.~(\ref{eq:frames}) does not depend on $Z, C, O$, we propose a simple adaptation of Proposition 1 from \cite{puny2022frame} to prove that $\F$ is G-equivariant and bounded for the space of materials $V$.

To show that $\F$ is translation equivariant, we need to show that $\mathcal{F}(\rho_1(g)D)=g \mathcal{F}(D)$ for all translations $g = (\mathbf{1}_D, \s) \in T(3)$. On one hand, the group product of these transformations is given by $g \F(D)=(\mathbf{1},\s)(U, \t) = (U, \t + \s)$. On the other hand, we compute below the frame of a translated version of $D$, written $D' = \rho_1(g)(D)$. Let $\t'$, $\Sigma'$, $U'$ denote respectively its centroid, its covariance matrix and the related matrix of eigenvectors.
\begin{align*}
    \rho_1(g)(D) &= (X + \1 \s^\top, Z, C + \1 \s^\top, O) \\
    \t' &= \frac{1}{n} (X + \1 \s^\top)^\top \1 = \frac{1}{n} X^\top \1 + \frac{1}{n} \s \1^\top \1 = \t + \s  \\
    \Sigma' &= \big( (X + \1 \s^\top) - 1 \t'^\top)^\top   (X + \1 \s^\top) - 1 \t'^\top) \big) \\
    & = (X - \1 t^\top)^\top (X - \1 t^\top) = \Sigma
\end{align*}
Hence, $\F(\rho_1(g)(D)) = \{ (U, \t + \s) | U = [\pm \mathbf{u}_1, \pm \mathbf{u}_2, \pm \mathbf{u}_3] \} = g \F(D)$. 

To show that $\F$ is rotation equvariant, we repeat the above with $g = (R, \mathbf{0}) \in SO(3)$. The group product of these transformations is $g \F(D) = (RU, R\t)$. On the other hand, we compute the frame of a rotated version of $D$, denoted $D'=\rho_1(g)(D)$. 
\begin{align*}
    \rho_1(g)(D) &= (XR^T, Z, CR^T, O) \\
    \t' &= \frac{1}{n} (XR^\top)^\top \1 = \frac{1}{n} R X^T \1 = R \t   \\
    \Sigma' &= (XR^T - \1 \t'^\top)^\top (XR^T - \1 \t'^\top) \\
    & = (XR^T - \frac{1}{n} \1(\1^\top XR))^\top (XR^T - \frac{1}{n} \1(\1^\top XR))^\top \\
    & = R(X-\t)^\top (X - \t)^\top = R \Sigma R^\top
\end{align*}
The eigendecomposition of $\Sigma'$ yields $\F(\rho_1(g)(D)) = \{ (RU, R\t) | [\pm \mathbf{u}_1, \pm \mathbf{u}_2, \pm \mathbf{u}_3] \} = g \F(D)$. As a result, $\F$ is E(3)-equivariant. 

$\F$ is also bounded as for compact $K \in \R^{n \times d}$, the translations are compact and therefore uniformly bounded for $X \in K$. Orthogonal matrices always satisfy $||R||_2= 1$. 

Next, an adaptation of Theorem 4 \cite{puny2022frame} proves that any arbitrary continuous equivariant function $\psi: V \rightarrow W$ approximable by a neural network $\Phi$ over $K_{\F}$ is approximable by $\langle\Phi\rangle_{\F}$. 
We let $c>0$ be a constant and consider an arbitrary $X \in K$. $\Psi: V \rightarrow W$ is an arbitrary G-equivariant function. 
\begin{align*}
    || \Psi(D) - \langle\Phi\rangle_{\F}(D)||_W &\leq \frac{1}{|\F(D)|} \sum_{g \in \F(D)} ||\rho_2(g) \Psi(\rho_1^{-1}(D)) - \rho_2(g)\Phi(\rho_1(g)^{-1}(D)) ||_W \\
    &\leq \max_{g \in \F(D)} || \rho_2(g)||_{op} || \Psi(\rho_1^{-1}(D)) - \Phi(\rho_1^{-1}(D)) ||_W \\
    &\leq c || \Psi - \Phi||_W    
\end{align*}

This completes the proof. 

\end{proof}

\textbf{Justification for new definition of $\rho_1$, $\rho_2$}. When $\text{E}(3)$ transformations are applied to materials with periodic crystals structures, one also needs to rotate the unit cell on which the structure is defined, in addition to the atom positions. Indeed, in OC20, crystal materials are defined as semi-infinite repeating substructures, and one therefore needs to define a unit cell. The graph's adjacency matrix is then constructed using periodic boundary conditions (pbc) as follows:
\begin{align}
    A_{ij} &= 
    \begin{cases}
        1 \text{ if } d_{ij} < c\\
        0 \text{ otherwise}
    \end{cases} \nonumber \\
    \text{with } d_{ij} &= || (\mathbf{x}_i - \mathbf{x}_j) + O_{ij*} \cdot C|| \text{ and cutoff radius } c.
\end{align}
So for two different rotations of the same atomic system $D$, say $D_1$ and $D_2$, frame averaging projects $D_1$ and $D_2$ to the same canonical position $\hat{X}$. However, their respective unit cell will be different. Since GNN predictions depend both on unit cell and atom positions, this will yield different predictions. In different terms, not applying the transformations $\rho_1$, $\rho_2$ on the unit cell too would break rotation equivariance, which was verified empirically. We therefore extended their definition to guarantee E(3)-equivariance for the space of materials with periodic crystal structures.

\textbf{Check that predictions are E(3)-equivariant}. Let $D_1$ and $D_2$ be two transformed versions of the same atomic system $D$ such that $\rho_1(g)D_2 = D_1$ for an arbitrary $g=(R, \mathbf{s}) \in E(3)$. 
$\F(D_1)$ is defined as in \Cref{subsec:frame-construction}, with canonical representation 
\begin{align*}
    \mathfrak{C}_1 &= \{\rho_1(g_1)^{-1}(D_1)| g_1 \in \F(D_1)\} \\
    \rho_1(g_1)^{-1}(D_1) &= \big((X - \mathbf{1}\mathbf{t}^\top)U, Z, (C - \mathbf{1}\mathbf{t}^\top)U ,O\big)
\end{align*}
$\F(D_2)$ is defined using similar computations to the proof above. We obtain centroid $\t'=R \t + \s$, covariance matrix $\Sigma'=R\Sigma R^T$ and eigenvectors $U' = R U$. This yields $\F(D_2) = \{ (U', \t') | U' = [\pm \mathbf{u'}_1, \pm \mathbf{u'}_2, \pm \mathbf{u'}_3] \}$ and corresponding canonical representation
\begin{align*}
    \mathfrak{C}_2 &= \{\rho_1(g_2)^{-1}(D_2)| g_2 \in \F(D_2)\} \\
    \text{with }\rho_1(g_2)^{-1}(D_2) &= \big( (X'-\mathbf{1} \t'^\top)U', Z, (\C-\mathbf{1} \t'^\top)U', O \big) \\
    &= \big( (X R^\top + \mathbf{1} \mathbf{s} - \mathbf{1} \t^\top R^\top - \mathbf{1} \mathbf{s})RU, Z, (\C-\mathbf{1} \t^\top)RU, O \big) \\
    &= \big( (X - \mathbf{1} \t^\top)U, Z, (\C-\mathbf{1} \t^\top)U, O \big) \\
    &= \rho_1(g_1)^{-1}(D_1)
\end{align*}
Hence, $\mathfrak{C}_1$ = $\mathfrak{C}_2$. Using \Cref{eq:equiv-fa-def}, GNN predictions are E(3)-equivariant. In the invariant case, $\langle\Phi\rangle_{\mathcal{F}}(D_1) = \langle\Phi\rangle_{\mathcal{F}}(D_2)$.


\section{FAENet}
\label{app:sec:faenet}

\subsection{Atom properties}
\label{app:subsec:atom-prop}

In the embedding block, to construct atom-level embeddings, we use the following properties from the \texttt{mendeleev} Python package \citep{mendeleev2014}:

\begin{enumerate}
    \item atomic radius 
    \item atomic volume
    \item atomic density
    \item dipole polarizability
    \item electron affinity
    \item electronegativity (allen)
    \item Van-Der-Walls radius
    \item metallic radius
    \item covalent radius
    \item ionization energy (first and second order).
\end{enumerate}


\subsection{Expressivity discussion}
\label{app:subsec:expressivity}

In this section, we discuss the expressive power of FAENet. Without making any hard claims, we believe our model is expressive, that is, able to distinguish between many different atomic systems up to global group actions. Note that FAENet enforces symmetries via the data by building upon the idea of Frame Averaging, which preserves the expressive power of the backbone architecture. This allows to construct a maximally expressive GNN without any architectural constraints. Next, FAENet processes directly relative atom information in addition to atom characteristic number (and other atomic properties), implying that it leverages enough information to uniquely identify each graph, which is not the case of several GNNs (SchNet, DimeNet, etc.). In the embedding block, since we use two-layer MLPs, which are universal approximators \citep{hornik1989multilayer}, we have the ability to learn injective functions mapping atomic number and geometric information to unique latent representations (provided that the dimensions of output space are large enough).

The expressive power of FANet is thus bound to the expressive power analysis of message passing itself, whose investigation is still an active topic of research. Despite GNNs revolutionising graph representation learning, there is limited understanding about their representation power and properties. Several works recently tried to fill in this gap, especially by comparing GNNs' expressive power to the (Geometric) Weisfeiler-Lehman isomorphism test \citep{joshi2022expressive, xu2018powerful, morris2019weisfeiler}. To achieve maximal expressive power, GNNs require an injective aggregation function, an injective update function as well as injective readout function. This would allow to construct unique representations for each graph, enabling to distinguish them. However, 
it is actually extremely challenging to satisfy these conditions. Also, theory does not correlate (always) with performance for geometric GNNs. 

Given the lack of theoretical guarantees about the expressive power of Geometric GNNs, \citet{joshi2022expressive} propose a series of experiments to assess empirically their ability to distinguish between very similar molecules, up to global E(3) transformations. In particular, the authors propose three experiments, where they train GNNs to map similar (but distinct) graphs to different labels. In particular, they try to:
\begin{enumerate}
    \item Distinguish \textit{k-chains}, which tests a model’s ability to propagate geometric information non-locally as well as its ability to overcome oversquashing with increased model depth; see \Cref{app:tab:kchains}.
    \item Distinguish \textit{rotationally symmetric structures}, which test a layer’s ability to identify neighbourhood orientation; see \Cref{app:tab:rotsym}.
    \item Distinguish counterexamples from \citep{pozdnyakov2022incompleteness}, which test a layer’s ability to create discriminating fingerprints for local neighbourhoods; see \Cref{app:tab:counterexamples}.
\end{enumerate}

In this paper, we reproduce their experiments to assess the expressivity of FAENet (GNN alone) and FAENet-SFA (when combined with our stochastic frame averaging framework). In the coming subsections, we only refer to FAENet because results are identical both for FAENet et FAENet-SFA, which empirically proves that SFA preserves the expressive power of FAENet (for these use cases). Baseline results were reproduced using the provided codebase \footnote{\url{https://github.com/chaitjo/geometric-gnn-dojo}}. 
We trained FAENet in the exact similar fashion, using standard hyperparameters.

\subsubsection{k-chains}

\begin{table}[h!]
    \centering
    \includegraphics[width=0.6\linewidth]{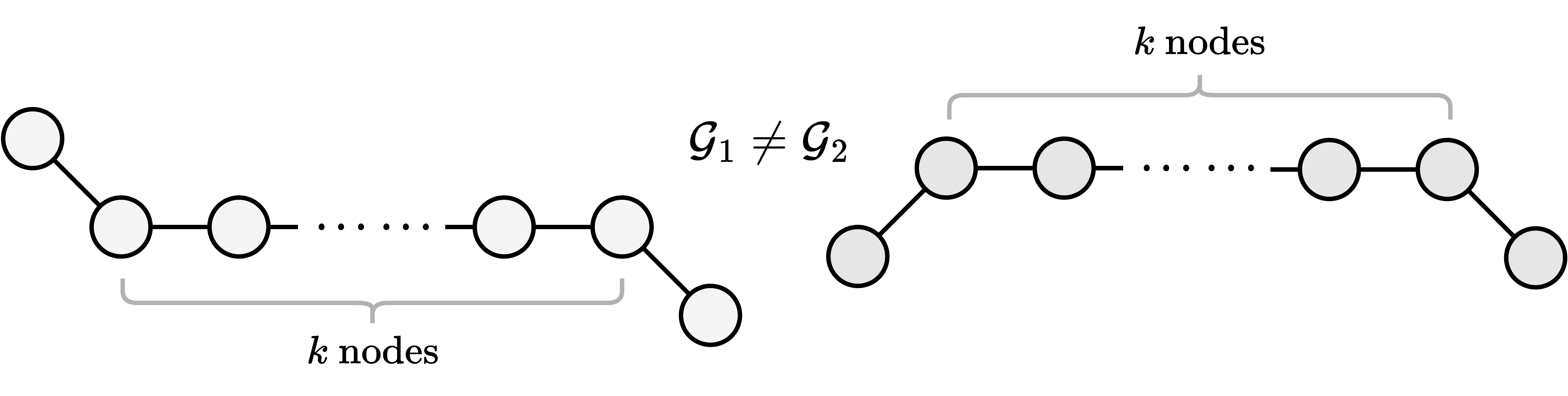}
    \begin{tabular}{clccccc}
        \toprule
        & ($k=\mathbf{4}$-chains) & \multicolumn{5}{c}{\textbf{Number of layers}} \\
        & \textbf{GNN Layer} & 1 & $\lfloor \frac{k}{2} \rfloor = 2$ & \cellcolor{gray!10} $\lfloor \frac{k}{2} \rfloor + 1 = \mathbf{3}$ & $\lfloor \frac{k}{2} \rfloor + 2$ & $\lfloor \frac{k}{2} \rfloor + 3$ \\
        \midrule
        \multirow{3}{*}{\rotatebox[origin=c]{90}{Inv.}} &
        \gray{IGWL} & \gray{50\%} & \gray{50\%} & \gray{50\%} & \gray{50\%} & \gray{50\%} \\
        & SchNet & 50.0 ± 0.00 & 50.0 ± 0.00 & 50.0 ± 0.00 & 50.0 ± 0.00 & 50.0 ± 0.00  \\
        & DimeNet & 50.0 ± 0.00 & 50.0 ± 0.00 & 50.0 ± 0.00 & 50.0 ± 0.00 & 50.0 ± 0.00 \\
        \midrule
        \multirow{5}{*}{\rotatebox[origin=c]{90}{Equiv.}} &
        \gray{GWL} & \gray{50\%} & \gray{50\%} & \gray{\textbf{100\%}} & \gray{\textbf{100\%}} & \gray{\textbf{100\%}} \\
        & E-GNN & 50.0 ± 0.00 & 50.0 ± 0.0 & \cellcolor{red!10} 50.0 ± 0.0 & \cellcolor{red!10} 50.0 ± 0.0 & \cellcolor{red!10} 50.0 ± 0.0  \\
        & GVP-GNN & 50.0 ± 0.00 & 50.0 ± 0.0 & \cellcolor{green!10} \textbf{100.0 ± 0.0} & \cellcolor{green!10} \textbf{100.0 ± 0.0} & \cellcolor{green!10} \textbf{100.0 ± 0.0} \\
        & TFN & 50.0 ± 0.00 & 50.0 ± 0.0 & \cellcolor{red!10} 50.0 ± 0.0 & \cellcolor{red!10} 50.0 ± 0.0 & \cellcolor{green!10} 80.0 ± 24.5 \\
        & MACE & 50.0 ± 0.00 & 50.0 ± 0.0 & \cellcolor{green!10} 90.0 ± 20.0 & \cellcolor{green!10} 90.0 ± 20.0 & \cellcolor{green!10} 95.0 ± 15.0\\
        & FAENet & \cellcolor{green!10} \textbf{100.0 ± 0.0} & \cellcolor{green!10} \textbf{100.0 ± 0.0} & \cellcolor{green!10} \textbf{100.0 ± 0.0} & \cellcolor{green!10} \textbf{100.0 ± 0.0} & \cellcolor{green!10} \textbf{100.0 ± 0.0}  \\
        & FAENet-SFA & \cellcolor{green!10} \textbf{100.0 ± 0.0} & \cellcolor{green!10} \textbf{100.0 ± 0.0} & \cellcolor{green!10} \textbf{100.0 ± 0.0} & \cellcolor{green!10} \textbf{100.0 ± 0.0} & \cellcolor{green!10} \textbf{100.0 ± 0.0} \\
        \bottomrule
    \end{tabular}
    \caption{\textit{$k$-chain geometric graphs.} 
    We train FAENet and FAENet-SFA with an increasing number of layers to distinguish $k=4$-chains. We report model accuracy for this binary graph classification task. An accuracy of 50 is equivalent to random predictions. We repeat the experiment 10 times and average results. Anomolous results are marked in \colorbox{red!10}{red} and expected results in \colorbox{green!10}{green}. Best results are in \textbf{bold}.
    }
    \label{app:tab:kchains}
\end{table}

\textbf{Experiment}. Here, we study FAENet's ability to incorporate and propagate geometric information beyond local neighbourhoods. To do so, we consider $k$-chain geometric graphs, each consisting of $k+2$ nodes where the $k$ central nodes are arranged in a line and the $2$ endpoints differ in their orientation. Because of the latter, every pair of $k$-chain graphs is different, up to E(3) transformations. In fact, they are $(\lfloor \frac{k}{2} \rfloor + 1)$-hop distinguishable, and $(\lfloor \frac{k}{2} \rfloor + 1)$ Geometric Weisfeiler-Lehman (GWL) iterations are theoretically sufficient to distinguish them (see \citep{joshi2022expressive} for more details). Indeed, both right-hand and left-hand parts of the chain are identical with respect to rotations, meaning that one can only distinguish between these graphs by considering these neighbourhoods together. In Table~\ref{app:tab:kchains}, we train FAENet along with $G$-equivariant and $G$-invariant GNNs (with an increasing number of layers) to distinguish $k$-chains. 

\textbf{Results}. Thanks to its novel functioning, free of all design constraints, FAENet enables to distinguish between these two $k$-chain graphs in a single message passing layer, while they are $(\lfloor \frac{k}{2} \rfloor + 1)=3$-hop distinguishable according to \citet{joshi2022expressive}. FAENet therefore outperforms all baselines with 1 or 2 GNN layers, including GWL algorithm. 
FAENet-SFA demonstrates that Stochastic FA enables preserves the expressive power of FAENet while enforcing symmetries via the data, as discussed deeply in the paper already. 
Aside from that, $G$-equivariant GNNs often require more iterations that prescribed by GWL and do not always attain perfect accuracy, pointing to preliminary evidence of oversquashing when geometric information is propagated across multiple layers using fixed dimensional feature spaces. IGWL and $G$-invariant GNNs are unable to distinguish $k$-chains for any $k \geq 2$ and $G = O(3)$.

\subsubsection{Rotationally symmetric structures}

\begin{table}[h!]
    \centering
    \includegraphics[width=0.5\linewidth]{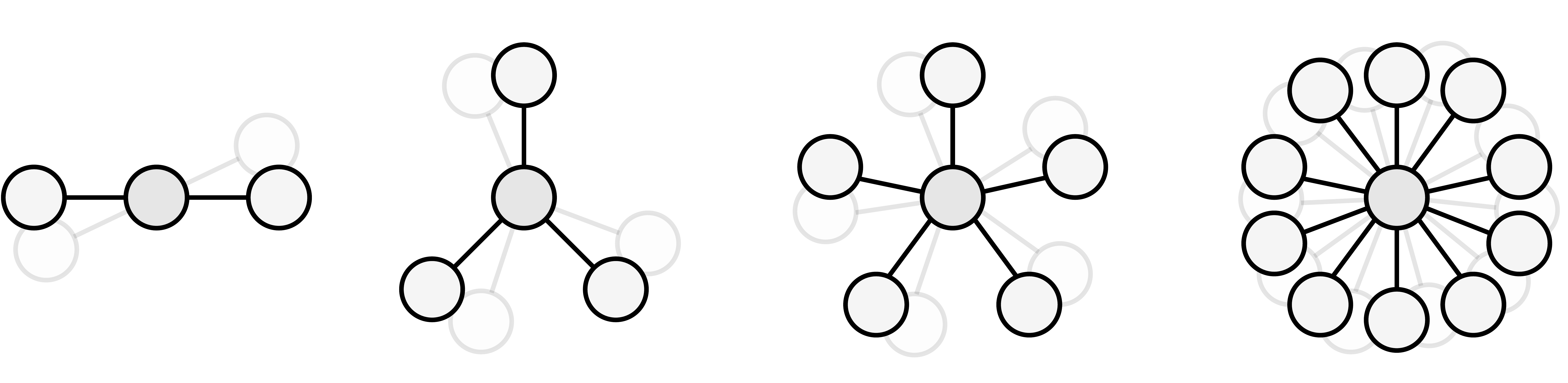}
    \begin{tabular}{clcccc}
        \toprule
        & & \multicolumn{4}{c}{\textbf{Rotational symmetry}} \\
        & \textbf{GNN Layer} & 2 fold & 3 fold & 5 fold & 7 fold \\
        \midrule
        \multirow{2}{*}{\rotatebox[origin=c]{90}{Scalar}}
        & FAENet$_{L=0}$ & \cellcolor{green!10} \textbf{100.0 ± 0.0} & \cellcolor{green!10} \textbf{100.0 ± 0.0} & \cellcolor{green!10} \textbf{100.0 ± 0.0} &  \cellcolor{green!10} \textbf{100.0 ± 0.0}  \\
        & FAENet-SFA$_{L=0}$ & \cellcolor{green!10} \textbf{100.0 ± 0.0} & \cellcolor{green!10} \textbf{100.0 ± 0.0} & \cellcolor{green!10} \textbf{100.0 ± 0.0} & \cellcolor{green!10} \textbf{100.0 ± 0.0} \\ 
        \midrule
        \multirow{2}{*}{\rotatebox[origin=c]{90}{Cart.}}
        & E-GNN$_{L=1}$ & \cellcolor{red!10} 50.0 ± 0.0 & 50.0 ± 0.0 & 50.0 ± 0.0 & 50.0 ± 0.0 \\
        & GVP-GNN$_{L=1}$ & \cellcolor{red!10} 50.0 ± 0.0 & 50.0 ± 0.0 & 50.0 ± 0.0 & 50.0 ± 0.0 \\
        \midrule
        \multirow{5}{*}{\rotatebox[origin=c]{90}{Spherical}}
        & TFN/MACE$_{L=1}$ & \cellcolor{red!10} 50.0 ± 0.0 & 50.0 ± 0.0 & 50.0 ± 0.0 & 50.0 ± 0.0 \\
        & TFN/MACE$_{L=2}$ & \cellcolor{green!10} \textbf{100.0 ± 0.0} & 50.0 ± 0.0 & 50.0 ± 0.0 & 50.0 ± 0.0 \\
        & TFN/MACE$_{L=3}$ & \cellcolor{green!10} \textbf{100.0 ± 0.0} & \cellcolor{green!10} \textbf{100.0 ± 0.0} & 50.0 ± 0.0 & 50.0 ± 0.0 \\
        & TFN/MACE$_{L=5}$ & \cellcolor{green!10} \textbf{100.0 ± 0.0} & \cellcolor{green!10} \textbf{100.0 ± 0.0} & \cellcolor{green!10} \textbf{100.0 ± 0.0} & 50.0 ± 0.0 \\
        & TFN/MACE$_{L=7}$ & \cellcolor{green!10} \textbf{100.0 ± 0.0} & \cellcolor{green!10} \textbf{100.0 ± 0.0} & \cellcolor{green!10} \textbf{100.0 ± 0.0} & \cellcolor{green!10} \textbf{100.0 ± 0.0} \\
        \bottomrule
    \end{tabular}
    \caption{\textit{Rotationally symmetric structures.} We train single layer $G$-equivariant GNN models (with order $L$ tensors) to distinguish two \emph{distinct} rotated versions of $L$-fold symmetric structures. We report model accuracy of this binary graph classification task, averaged over 10 runs. Anomolous results are marked in \colorbox{red!10}{red} and expected results in \colorbox{green!10}{green}. Best in \textbf{bold}. 
    }
    \label{app:tab:rotsym}
\end{table}

In this experiment, we try to evaluate FAENet's ability to discriminate the orientation of structures with rotational symmetry. To do so, \citet{joshi2022expressive} propose to train geometric GNNs to assign two different rotations of some L-fold symmetric structures (attached to an existing system) to different classes. 

\textbf{Results} are displayed in \Cref{app:tab:rotsym}. Although this experiment was designed to illustrate the utility of higher-order tensors in G-equivariant GNNs, FAENet shows outstanding performance for a method using scalar vectors ($L=0$) only. 
Methods using cartesian vectors cannot even discriminate two-folds structures; and methods with spherical tensors of order $L$ are unable to identify the orientation of structures with rotation symmetry greater than $L$-fold, which suggests high computational cost. Here again, FAENet appears as the most desirable method, combining simple functioning, scalability and expressivity.

\subsubsection{Edge cases}

\begin{table}[h!]
    \centering
    \begin{tabular}{clccc}
        \toprule
        & & \multicolumn{3}{c}{\textbf{Edge cases}} \\
        & \multirow{1}{*}{\textbf{GNN Layer}} & 2-body & 3-body & 4-body \\
        \midrule
        \multirow{2}{*}{\rotatebox[origin=c]{90}{Inv.}} 
        & SchNet$_{\text{2-body}}$ & \cellcolor{red!10} 50.0 ± 0.0 & 50.0 ± 0.0 & 50.0 ± 0.0 \\
        & DimeNet$_{\text{3-body}}$ & \cellcolor{green!10} \textbf{100.0 ± 0.0} & 50.0 ± 0.0 & 50.0 ± 0.0 \\
        \midrule
        \multirow{6}{*}{\rotatebox[origin=c]{90}{$O(3)$-Equiv.}} 
        & E-GNN$_{\text{2-body}}$ & \cellcolor{red!10} 50.0 ± 0.0 & 50.0 ± 0.0 & 50.0 ± 0.0 \\
        & GVP-GNN$_{\text{3-body}}$ & \cellcolor{green!10} \textbf{100.0 ± 0.0} & 50.0 ± 0.0 & 50.0 ± 0.0 \\
        & TFN$_{\text{2-body}}$ & \cellcolor{red!10} 50.0 ± 0.0 & 50.0 ± 0.0 & 50.0 ± 0.0 \\
        & MACE$_{\text{3-body}}$ & \cellcolor{green!10} \textbf{100.0 ± 0.0} & 50.0 ± 0.0 & 50.0 ± 0.0 \\
        & FAENet$_{\text{2-body}}$ & \cellcolor{green!10} \textbf{100.0 ± 0.0} & \cellcolor{green!10} \textbf{100.0 ± 0.0} & \cellcolor{green!10} \textbf{100.0 ± 0.0} \\
        & FAENet-SFA$_{\text{2-body}}$ & \cellcolor{green!10} \textbf{100.0 ± 0.0} & \cellcolor{green!10} \textbf{100.0 ± 0.0} & \cellcolor{green!10} \textbf{100.0 ± 0.0} \\
        \bottomrule
    \end{tabular}
    \caption{\textit{Counterexamples from \citep{pozdnyakov2022incompleteness}.}
    We train geometric GNNs at distinguishing edge cases structures, which are (supposedly) indistinguishable using $k$-body scalarisation. We repeat the experiment 10 times and report the accuracy. FAENet manages to distinguish them all while presenting 2-body order only.  
    }
    \label{app:tab:counterexamples}
\end{table}

For this experiment, we look at three specific type of graphs proposed by \citep{pozdnyakov2022incompleteness} to illustrate the limitations of existing geometric GNNs. Each counterexample consists of a pair of local neighbourhoods that are indistinguishable when comparing their set of k-body scalars\footnote{Body-order of scalarisation: number of nodes involved in computing local invariant scalars.}. The 3-body and 4-body counterexamples correspond respectively to Fig.1(b) and Fig.2(e) in \citep{pozdnyakov2022incompleteness}, while the 2-body counterexample corresponds to the running example of \citep{joshi2022expressive}. The aim of the task is the same, meaning training a single layer GNN to create discriminating fingerprints for local neighbourhoods. 

\textbf{Results} are reported in \Cref{app:tab:counterexamples}. While \citet{joshi2022expressive} concluded that Geometric GNN layers with body order $k$ cannot distinguish the corresponding k-body counterexample, FAENet does enable to distinguish the $4$-body counterexample while being only a $2$-body method\footnote{Actually, since FAENet utilises non-linearities in its message passing, the concept of k-body order looses relevance. FAENet could be considered as a many-body method, demonstrating the advantage of its frame based paradigm.}. This demonstrates its expressivity and the strength of the novel paradigm proposed (i.e. unconstrained GNN functioning; symmetries enforced via the data without loosing expressive power). 



\subsection{Ablation Study on IS2RE}
\label{app:subsec:ablation-study}

\begin{table}[h!]
\centering
\caption{Ablation study of FAENet -- we study the impact of various components on the model performance-scalability trade-off. Scalability is measured via the throughput (the number of processed samples per seconds) while performance is measured by the Energy MAE, averaged across all four validation splits.  
\\}
\label{tab:ablation-study}
\begin{tabular}{l|cc}
\toprule
\textit{Component} & \textit{Throughput} & \textit{MAE} \\ 
 & (samples/s) & (average) \\
\midrule
\textbf{FAENet} (baseline)                 & \textbf{2129}         &    \textbf{568} \\
\textit{\indent - Only} $\vec{r}_{ij}$  & 2390            & 603  \\
\textit{\indent - Only Spherical Harmonics}   & 1441             & 626 \\
\textit{\indent - All}   & 1560             & 585    \\
\textit{\indent - $2^{nd}$ layer in MLP}         & 1828           & 588  \\
\textit{\indent - Simple MP} & 2706      & 606 \\
\textit{\indent - Basic MP} & 3065      & 629 \\
\textit{\indent - Updownscale MP} & 1505      & 592 \\
\textit{\indent - Attention mecha} & 1929      & 712 \\
\textit{\indent - No complex MP}           & 2265              & 586  \\
\textit{\indent - Simple Energy Head}            & 2224            & 589  \\
\textit{\indent - No Jumping Connections}         & 2038             & 582  \\
\bottomrule
\end{tabular}
\end{table}

In this section, we perform a thorough ablation study of FAENet on the OC20 IS2RE dataset. Even if we motivate a simple and efficient GNN architecture, we propose in this paper a new paradigm for incorporating geometric information into atom representations (since we construct a model free of all symmetry-preserving constraints). We therefore compared various relevant design choices with respect to the embedding block, the message passing scheme and the output block. The results displayed in \Cref{tab:ablation-study} empirically validate FAENet's current architecture. In particular, we tried the following ideas and drew the corresponding conclusions:
\begin{itemize}
    \item \textit{Only} $\vec{r}_{ij}$ involves using $\vec{r}_{ij}$ alone in the embedding block to derive each edge representation, instead of $\vec{r}_{ij} || RBF(d_{ij})$. Please refer to \Cref{eq:edge-embedding}. \\
    $\rightarrow$ Using a radial basis function (RBF) of distance information in addition to $r_{ij}$ improves performance significantly. Running time is a bit higher, obviously, but not enough to outweigh the performance gains. 
    \item \textit{Only Spherical harmonics}: we replace $\big(\Vec{r}_{ij} || \text{RBF}(d_{ij})\big)$ by spherical harmonics (of order up to $3$) on $\dfrac{\Vec{r}_{ij}}{||\Vec{r}_{ij}||}$ in \Cref{eq:edge-embedding}. Spherical harmonics form a basis for irreducible representations in SO(3) and were shown to be very informative quantities by previous works \citep{frank2022so3krates, thomas2018tensor, batatia2022mace}. \\
    $\rightarrow$ We show that using directly a 2-layer MLP of relative information is more attractive than (higher-order) spherical harmonics. 
    \item \textit{All}. We concatenate spherical harmonics to directional information and distance information ($\vec{r}_{ij}$, $d_{ij}$ and SH). \\
    $\rightarrow$ We show that adding spherical harmonics in addition to directional information does not provide additional benefits for this task.  
    \item \textit{2-layer-MLP}: we replace the two-layer-MLP of the Embedding block by a one-layer-MLP, see \Cref{eq:edge-embedding}. \\
    $\rightarrow$ No clear conclusion here. Reducing MLP to 1-layer hurts performance but improves slightly compute time. We favoured performance here. 
    \item \textit{Basic MP}: we do not transform the edge embedding inside the interaction block. $\f_{ij}^{(l)} = \e_{ij}$ \\
    $\rightarrow$ This message passing form is much more scalable (obviously) but significantly less efficient. This proves the relevance of refining the filters using edge information and atom information at each layer. 
    \item \textit{Simple MP}: we do not compute the filter using node features $\h_i$ and $\h_j$. In mathematical form, we use $\f_{ij}^{(l)} = \sigma (\text{MLP}(\e_{ij}))$ instead of $\f_{ij}^{(l)} = \sigma (\text{MLP}(\e_{ij}||\h_i^{(l)}||\h_j^{(l)}))$. \\
    $\rightarrow$ This Message Passing form is shown to be more scalable but not as efficient. Using an edge's endpoints helps learn more relevant filters.  
    \item \textit{Updownscale MP}: we additionally transform node embeddings inside each interaction block, before the continuous convolution. \\
    $\rightarrow$ This does not show increased performance, validating current design.   
    \item \textit{Attention mecha}: we weight each message by an attention coefficient, learned using neighbour's feature vector, node's feature vector and edge embedding. This is expressed as $h_i^{(l+1)}= h_i^{(l)} + \sum_{j \in N_i} \alpha_{ij} W h_j^{(l)}$ with $\alpha_{ij}= Att(h_i, h_j, \vec{r}_{ij}) = h_i^T [e_{ij} \odot h_j] / \sqrt{F})$. \\
    $\rightarrow$ It does not yield good results in practice. Current filter construction is more efficient. 
    \item \textit{No complex MP}: we use two-layers MLP instead of one-layer MLP in each Interaction block of the model (see Fig.~\ref{fig:gnn-pipeline}). \\
    $\rightarrow$ This was not found beneficial, both in terms of performance and running time. 
    \item \textit{Simple Energy Head}: involves replacing the current weighted average of atom's final representation to compute graph-level prediction (in the Output Block) by the standard global pooling approach. \\
    $\rightarrow$ The weighted average is found (slightly) beneficial.
    \item \textit{No jumping connections}: We remove the jumping connections from each Interaction Block to the Output block. \\
    $\rightarrow$ Jumping connections are found (slightly) beneficial.  
\end{itemize}


\section{Extended Results}
\label{app:sec:extended-results}


\subsection{Experimental setup -- Software and Hardware}
\label{app:subsec:experimental-settings}

We would like to acknowledge and thank authors of the following Python libraries that we have used to realise this work, in particular Pytorch v.1.13~\cite{NEURIPS2019_bdbca288}, PyTorch Geometric v2.2.0~\cite{Fey/Lenssen/2019} NumPy v1.23.5~\cite{harris2020array} and authors of the OCP Github repository which we used as a starting point.

Experiments were run on 1 NVIDIA RTX8000 GPUs. We would like to also thank our lab's IT staff for putting together and maintaining our computing infrastructure as well as for supporting researchers.


\subsection{Hyper Parameters}

We detail FAENet's list of hyperparameters for all four datasets IS2RE, S2EF, QM7-X and QM9 in Table \ref{tab:app:hps}. 

\begin{table}[h]
\centering
\caption{Model and optimization hyper parameters for FAENet used in experiments reported. }
\label{tab:app:hps}
\begin{tabular}{lcccc}
\textit{}          & \textit{OC20-IS2RE} & \textit{OC20-S2EF} & \textit{QM7X}     & \textit{QM9}      \\ \hline
\textit{Activation function}        & swish & swish   & swish  & swish \\
\textit{Cutoff distance}            & 6     & 6       & 5      & 6     \\
\textit{Force head hidden channels} & -     & 256     & 256    & -     \\
\textit{Hidden channels}            & 384   & 256     & 500    & 400   \\
\textit{Max number of neighbors}    & 40    & 30      & 40     & 30    \\
\textit{Number of filters}          & 480   & 480     & 400    & 480   \\
\textit{Number of gaussians}        & 104   & 136     & 50     & 100   \\
\textit{Interaction blocks}         & 5     & 7       & 5      & 5     \\
\textit{Period \& Group channels}    & 64    & 64      & 32     & 32    \\
\textit{Tag channels}               & 64    & 32      & -      & -     \\ \hline
\textit{Optimizer}                  & AdamW & AdamW   & AdamW  & AdamW \\
\textit{Scheduler} & Cosine Annealing    & Cosine Annealing   & Reduce On Plateau & Reduce On Plateau \\
\textit{Warmup steps}               & 6,000  & 30,000   & 3,000   & 3,000  \\
\textit{Learning Rate}              & 0.002 & 0.00025 & 0.0002 & 0.001 \\
\textit{Batch size}                 & 256   & 192     & 100    & 64    \\
\textit{Energy loss coeff.}         & -     & 1       & 1      & -     \\
\textit{Force loss coeff.}          & -     & 100     & 100    & -     \\
\textit{Energy grad. loss coeff.}   & -     & 5       & 15     & -     \\
\textit{Steps}                      & 22K   & 85K     & 1.4M   & 1.5M  \\ \hline
\end{tabular}
\end{table}


\subsection{QM9 Dataset}
\label{subsec:app:qm9-dataset}

\begin{table*}[h]
    \begin{center}
        \resizebox{\textwidth}{!}
        {\begin{tabular}{ll|ccccccccc}
        \toprule
        Property                                        &Unit      &SchNet &PhysNet &MGCN &DimeNet &DimeNet++& SE(3)-T&SphereNet &ComENet &FAENet\\
        \midrule
        $\mu$                                           &D         &0.033 &0.0529   &0.0560 &0.0286 &0.0297 & 0.051 &0.0245     & 0.0245     & 0.0289\\
        $\alpha$                                        &${a_0}^3$ &0.235 &0.0615   &0.0300 &0.0469 &0.0435 & 0.142 &0.0449     & 0.0452     & 0.0527\\
        $\epsilon_\text{HOMO}$                          &meV       &41    &32.9     &42.1   &27.8   &24.6   & 35.0 &22.8       & 23.1       & 24.5   \\
        $\epsilon_\text{LUMO}$                          &meV       &34    &24.7     &57.4   &19.7   &19.5   & 33.0  &18.9       & 19.8       & 20.0 \\
        $\Delta\epsilon$                                &meV       &63    &42.5     &64.2   &34.8   &32.6   & 53.0 & 31.1       & 32.4       & 40.2 \\
        $\left< R^2 \right>$                            &${a_0}^2$ &0.073 &0.765    &0.110  &0.331  &0.331  & - &0.268      & 0.259      & 0.498 \\
        ZPVE                                            &meV       &1.7   &1.39     &1.12   &1.29   &1.21   & - &1.12       & 1.20       & 1.26 \\
        $U_0$                                           &meV       &14    &8.15     &12.9   &8.02   &6.32   & - &6.26       & 6.59       & 6.79 \\
        $U$                                             &meV       &19    &8.34     &14.4   &7.89   &6.28   & - &6.36       & 6.82       & 6.80  \\
        $H$                                             &meV       &14    &8.42     &14.6   &8.11   &6.53   & - &6.33       & 6.86       & 6.74  \\
        0$G$                                            &meV       &14    &9.4      &16.2   &8.98   &7.56   & - &7.78       & 7.98       & 7.91 \\
        $c_\text{v}$ &$\frac{\mbox{cal}}{\mbox{mol K}}$            &0.033 &0.028    &0.038  &0.025  &0.023  & 0.052 &0.022      & 0.024      & 0.023\\
        \midrule
        $R_{p=1..11}$                                   & mean     & -    & 0.247   &-0.021 &0.366  &0.419  & - &0.445      &0.423       &0.395\\
                                                        & std      & -    &0.331    &0.421  &0.208  &0.226  & - &0.205      &0.205       &0.212\\
        \midrule
        Time                                           & s        &98   &  -      & -      &582   &  330   & - &345       &224          &87    \\
        \bottomrule        \end{tabular}}
        \caption{Comparisons between FAENet and baseline models on QM9 dataset in terms of (1) MAE, (2) \textit{relative improvement} $R_{p=1..11}$ with respect to SchNet, and (3) Time, which we measured for 1 epoch of training using the same hardware and software.
        }
        \label{tb:qm9-results}
    \end{center}
\end{table*}

We provide the detailed list of per-model property MAE performance in Table~\ref{tb:qm9-results}. The results reported in are comptued as followed. For each property $p$ we compute the \textit{relative improvement} with respect to SchNet as:

\begin{equation}
    R_p^{\text{Model}} = \frac{\text{MAE}_p^{\text{SchNet}} - \text{MAE}_p^{\text{Model}}}{\text{MAE}_p^{\text{SchNet}}}.
\end{equation}

We then compute and report in Fig.~\ref{fig:qm9-pareto} the mean and standard deviation of the series $R_{p=0..11}^{\text{Model}}$. Note that we treat property $\left< R^2 \right>$ as an outlier and exclude it from the aforementioned statistics because SchNet is oddly the best by a very large margin and we have not been able to reproduce such good results. In addition, in spite of our effort to re-run baselines ourselves, we could not re-implement PhysNet \citep{unke2019physnet}, MGCN \citep{wang2020mgcn} nor SE(3)-Transformers \cite{fuchs2020se} in our pipeline.


\subsection{OC20 full table of results}
\label{app:subsec:full-tables}

We report the full results for OC20 IS2RE in Table \ref{tb:result_oc20_is2re_full} and for S2EF 2M in Table \ref{tb:result_oc20_s2ef_full}.

\begin{table*}[h]
    \begin{center}
        \caption{Exhaustive table of results on \textsl{OC20 IS2RE} ``All'' dataset, for all 4 validation splits (ID, OOD Ads, OOD Cat, OOD Both). We measure performance in terms of energy MAE and the percentage of Energy within Threshold (EwT) of the ground truth energy. We average results across the four val splits (Average). The best performance is shown in bold and the second best is shown with underlines. Scalability is measured with training time for one epoch (train, in minutes) and inference throughput (infer, number of samples processed in a second).
        }
    \label{tb:result_oc20_is2re_full}
    \resizebox{\textwidth}{!}
    {\begin{tabular}{l | cc | ccccc | ccccc  }
    \toprule
    & \multicolumn{2}{c|}{Time } &\multicolumn{5}{c|}{Energy MAE [eV] $\downarrow$} & \multicolumn{5}{c}{EwT $\uparrow$}  \\
    \cmidrule(l{4pt}r{4pt}){2-3}
    \cmidrule(l{4pt}r{4pt}){4-8}
    \cmidrule(l{4pt}r{4pt}){9-13}
 Model &Train $\downarrow$ &Infer. $\uparrow$ & ID &  OOD Ads & OOD Cat & OOD Both &Average& ID &  OOD Ads & OOD Cat & OOD Both & Average\\
\midrule
SchNet     & \underline{9min} &  \underline{3597} & 0.6372             & 0.7342             & 0.6611             & 0.7035             & 0.6840             & 2.96\%             & 2.22\%             & 3.03\%             & 2.38\%             & 2.65\% \\ 
DimeNet++  & 170min           & 115             & 0.5716             & 0.7224             & 0.5612             & 0.6615             & 0.6283             & 4.26\%             & 2.06\%             & 4.10\%             & 3.21\% & 3.40\% \\ 
SphereNet  & 290min           & 83              & 0.5632             & 0.6682             & 0.5590             & 0.6190             & 0.6023             & 4.56\%             & 2.70\%             & 4.59\%             & 2.70\%             & 3.64\% \\
EGNN       & -                & -                 & 0.5497             & 0.6851             & 0.5519             & 0.6102             & 0.5992             & \underline{4.99}\% & 2.50\%             & \underline{4.71}\% & 2.88\%             & 3.77\% \\
PaiNN & -                & -                 & 0.5781             & 0.7037             & 0.5701             & 0.6139             & 0.6164             & 4.31\% & 2.60\%             & 4.35\% & 2.74\%             & 3.50\% \\
Forcenet   & 120min           & 157                & 0.6582             & 0.7017             & 0.6323             & 0.6285             & 0.6551             & 3.14\%             & 2.47\%             & 2.58\%             & 2.83\%             & 3.13\% \\ 
SpinConv   & -                & -                 & 0.5583             & 0.7230             & 0.5687             & 0.6738             & 0.6309             & 4.08\%             & 2.26\%             & 3.82\%             & 2.33\%             & 3.12\% \\
GemNet-T   & 200min           & 104              & 0.5561             & 0.7342             & 0.5659             & 0.6964             & 0.6382             & 4.51\%             & 2.24\%             & 4.37\%             & 2.38\%             & 3.38\% \\
ComENet    & 20min            & 416              & 0.5558             & 0.6602             & \underline{0.5491} & 0.5901             & 0.5888             & 4.17\%             & 2.71\%             & 4.53\%             & 2.83\%             & 3.56\% \\
GNS        & -                & -                 & \underline{0.5400} & 0.6500             & 0.5500             & 0.5900             & 0.5825             & -                  & -                  & -                  & - - \\
Graphormer & -                & -                 & \textbf{0.4000}    & \textbf{0.5700}    & \textbf{0.4200}    & \textbf{0.5000}    & \textbf{0.4725}    & \textbf{8.97}\%    & \textbf{3.45}\%    & \textbf{8.18}\%    & \textbf{3.79} \%   & \textbf{6.09} \% \\
FAENet     & 12min            & 2469              & 0.5446             & \underline{0.6115} & 0.5707             & \underline{0.5449} & \underline{0.5679} & 4.46\%             & \underline{2.95}\% & 4.67\%             & \underline{3.01}\% & \underline{3.78}\% \\
\bottomrule
\end{tabular}}
\end{center}
\end{table*}

\begin{table*}[h]
    \begin{center}
        \caption{Exhaustive table of results on \textsl{OC20 S2EF} ``2M'' dataset, for all 4 validation splits (ID, OOD Ads, OOD Cat, OOD Both). This table expands on \Cref{tb:result_oc20_s2ef}. We measure performance in terms of energy MAE and forces MAE. We average results across the four val splits (Average). The best performance is shown in bold and the second best is shown with underlines. Scalability is measured with training time for one epoch (train, in minutes) and inference throughput (infer, number of samples processed in a second).
        }.
    \label{tb:result_oc20_s2ef_full}
    \resizebox{\textwidth}{!}
    {\begin{tabular}{l | cc | ccccc | ccccc  }
    \toprule
    & \multicolumn{2}{c|}{Time } &\multicolumn{5}{c|}{Energy MAE [eV] $\downarrow$} & \multicolumn{5}{c}{Force MAE [meV] $\downarrow$}  \\
    \cmidrule(l{4pt}r{4pt}){2-3}
    \cmidrule(l{4pt}r{4pt}){4-8}
    \cmidrule(l{4pt}r{4pt}){9-13}
Model     &Train $\downarrow$ &Infer.$\uparrow$ & ID       &  OOD Ads & OOD Cat & OOD Both &Average             & ID    & OOD Ads & OOD Cat & OOD Both &Average\\
\midrule
SchNet    & 98min             & 509             & 1.4075   & 1.5228   & 1.4383  & 1.5982   & 1.4917             & 76.3  & 86.4    & 79.7    & 90.1     & 83.1 \\
DimeNet++ & 1157min           &  52             & 0.7438   & 0.8425   & 0.7788  & 0.8734   & 0.8096             & 57.3  & 68.6    & 66.0    & 73.5     & 66.3 \\
ForceNet  & 1476min                & 36              & 0.6757       & 0.7667        & 0.6650       & 0.8987        &  0.7548                 &  55.3   &  61.6     &  72.7     &  54.8       & 61.0 \\
GemNet-OC & --                & 18              & -        & -        & -       & -        & \textbf{0.2860}    & --    & --      & --      & --       & \textbf{25.7} \\
FAENet    & 75min             & 623             & 0.3960   & 0.4548   & 0.4295  & 0.5766   & \underline{0.4643} & 51.1 & 58.9   & 50.7   & 69.4    & \underline{57.5} \\
\bottomrule
\end{tabular}}
\end{center}
\end{table*}


\section{Empirical Evaluation of Model Properties}
\label{app:sec:empirical-eval}

In this section, we verify the correctness of our symmetry preserving framework, including all different theoretical properties.

\subsection{Methods and Metrics}
\label{app:sec:methods-metrics}

For this experiment, we look at different methods: (1) FAENet combined with full frame averaging (Full-FA); (2) FAENet with Stochastic Frame Averaging (SFA); (3) FAENet with Stochastic Frame Averaging for the SE(3) group instead of E(3), i.e. less frames to sample from; (4) FAENet with data augmentation (DA); (5) FAENet alone (No-FA); (6) untrained FAENet (No-Train); (7) ForceNet alone (8) SchNet alone. For each method, we study some targeted equivariant and invariant properties via the following metrics:
\begin{itemize}
    \item \textit{Pos}. To check the correctness of our implementation of frame averaging, we let \textit{Pos} be a binary variable which takes the value 1 if, for arbitrary datapoints $D_1, D_2$ such that $D_2 = \rho_1^{(-1)}(g)(D_1)$ for any $g\in E(3)$, $\mathfrak{C}_1=\mathfrak{C}_2$, i.e. they lead to the same representation in the projected space; and 0 otherwise. By definition, it shall be the case for \textit{Full FA} and not for the others. For methods which don't use FA, we set $pos=0$ if $D_1$ and $D_2$ are mapped to the same representation. 
    \item \textit{Rot-I} and \textit{Refl-I}. We measure the rotation invariance \textit{Rot-I} and reflection invariance \textit{Refl-I} property of FAENet by computing the difference in GNN prediction between every samples $D_1$ (of the ID val split) and $D_2$ defined as a $\text{SO}(3)$ transformation of $D_1$, as above. We also compute the percentage variation in prediction with respect to energy (when rotated), with \textit{\%-diff}: $100 \times \frac{\textit{Rot-I}}{y_{true}}$. 
    \item \textit{F-Rot-E} and \textit{F-Refl-E}. In a similar fashion, we evaluate force predictions's rotation and reflection equivariance to $\text{E}(3)$ transformations, denoted respectively \textit{F-Rot-E} and \textit{F-Refl-E}. 
\end{itemize}

\subsection{Results}

For the IS2RE and S2EF datasets, 
we display the obtained results in Table~\ref{table:app:prop-check-is2re} and in Table ~\ref{table:app:prop-check-s2ef}. From them, we draw the following conclusions:
\begin{itemize}
    \item \textbf{Full FA} yields near-perfect invariant/equivariant model predictions, proving the correctness of our implementation. Note that metrics are not exactly $0$ for Full FA (and Schnet) due to software precision limitations.
    \item \textbf{Stochastic Frame Averaging (SFA)} produces a great approximation to equivariant/invariant predictions: \textit{Rot-I}, \textit{F-Rot-E}, \textit{Refl-I} and \textit{F-Refl-E} get close to 0 (the scale is in meV, not eV), \textit{\%-diff} is below $2\%$. Invariant metrics on IS2RE are divided by a factor of $\approx8$ for SFA compared to No-FA or ForceNet, showing that our approach learns efficiently to enforce symmetries. In different terms, energy predictions are much closer for two rotated (or reflected) versions of the same graph when using SFA compared to DA or No-FA. Besides, SFA is our best performing model. When comparing to Full FA (or even other equivariant/invariant models \cref{tb:result_oc20_is2re_full}), SFA demonstrates that perfectly enforcing symmetries does not necessarily lead to better performance. Presenting a setting that makes it easy for the model to learn to accomplish the required task seems essential too. 
    \item \textbf{$\text{SE}(3)$-SFA} approximation yields slightly better metrics because FAENet has to learn data symmetries from less frames (cf SFA), which is easier. By construction, it is less invariant/equivariant to reflections. However, this does not seem to impact performance for this application.
    \item \textbf{Data augmentation (DA)} achieves worse symmetry metrics and MAE score on both datasets, with similar run time (the PCA computation is negligible for FA). Equivariance seems harder to learn than invariance for DA, compared to SFA. For instance, on S2EF, rotation equivariance for DA decreases by $60\%$ compared to SE(3)-SFA while performance drops by $9\%$. We can expect even bigger differences on other datasets since the $z$-axis is fixed in OC20: the adsorbate is always on top of the catalyst. This means that we only look at 2D rotations/reflections, making it easier to learn symmetries via data augmentation. SFA could be (even) more useful on ``pure'' 3D cases. This being said, DA gives very satisfying results on this dataset, confirming that enforcing symmetries implicitly via the data is an attractive (and under-exploited) approach in materials modeling. 
    \item \textbf{No-FA}, i.e. not using any kind of symmetry preservation techniques, does not lead to invariant/equivariant predictions (e.g. invariance metrics are roughly 8 times higher than for SFA) although these metrics improve significantly compared to  \textbf{No-Train} on S2EF. This means that FAENet learn symmetry to some extent, even when not directly encouraged to do so. No-FA showcases a ``significant'' decrease in performance (569 to 595, 464 to 522), proving that enforcing symmetries to some extent still seem desirable. Again, equivariance seems harder than invariance to learn for non-symmetry preserving methods. Again, we have 2D symmetries and expect this gap to be higher on QM9/QM7-X.
    \item Repeating these experiments on \textbf{ForceNet} for IS2RE confirms the trend observed for FAENet: SFA is more efficient than DA, both in terms of performance and in terms of enforcing equivariance/invariance. It also reveals that the design of FAENet is superior, probably because it leverages more efficiently geometric information. 
\end{itemize}

\begin{table}[h!]
\centering
\caption{Test of FAENet's theoretical properties on IS2RE depending on the different symmetry-preservation techniques used. We report several metrics (see \ref{app:sec:methods-metrics}) measuring perfect symmetry-preservation (Pos), rotation invariance (Rot-I), reflection invariance (Refl-I), percentage variation in prediction when rotated (\%-diff) and energy average validation MAE. Everything (but Pos) is given in meV.
Note that we ideally want to minimise these metrics.}
\label{table:app:prop-check-is2re}
\begin{tabular}{lccccc}
\toprule
\textit{Model/Metrics} & Pos & Rot-I & \%-diff & Refl-I & MAE \\
\midrule
Full FA     & 0  & 0.07      & 0.26  & 0.05   & 578    \\
SFA  & 1  & 6.57   & 2.12   & 6.54      & 569    \\ 
SE(3)-SFA    & 1  & 6.35  & 1.87    & 8.90         & 567  \\ 
DA          & 1  & 7.41  & 2.97    & 8.22         & 582   \\
No-Train & 1  & 21.8  & 6.03  & 24.9         & 595 \\
No-FA       & 1  & 51.1  & 18.3  & 55.3         & 595 \\ 
SchNet      & 0  & 0.00 & 0.00     & 0.00       & 684 \\
ForceNet      & 1  & 61.6 & 14.8     & 63.0       & 655 \\
Forcenet-DA      & 1  & 18.9  & 7.47     & 20.8   & 683 \\
Forcenet-SFA & 1 & 7.11 & 2.59 & 7.14 & 638 \\
\toprule
\end{tabular}
\end{table}

\begin{table*}[h!]
\centering
\caption{Test FAENet's theoretical properties on S2EF depending on the different symmetry-preservation methods used. We report several metrics (see \ref{app:sec:methods-metrics}) measuring implementation correctness (\textit{Pos}), energy rotation invariance (\textit{Rot-I}), energy reflection invariance (\textit{Refl-I}), force reflection equivariance (\textit{F-Refl-E}), force rotation equivariance (\textit{F-Rot-E}) and energy MAE. Everything is in meV.
}
\label{table:app:prop-check-s2ef}
\begin{tabular}{lccccccc}
\toprule
\textit{Model/Metrics} & Pos & Rot-I & Refl-I  & F-Rot-E  & F-Refl-E  & MAE \\ 
\midrule
Full FA        & 0  & 0.05          & 0.03           & 0.07    & 0.05  & 489    \\ 
Stocha. FA     & 1  & 6.94      & 6.96      & 4.79   & 4.78   & 464    \\ 
SE(3)-SFA       & 1  & 6.23      & 8.86      & 3.78  & 5.25   & 478  \\  
DA            & 1  & 8.88      & 9.12         & 6.12  & 6.46    & 504   \\
No-FA        & 1  & 26.81     & 27.61        & 10.27  & 10.12  & 522 \\  
No-train       & 1 & 86.03  & 87.09     & 32.31       & 31.92  & -- \\ 
SchNet        & 0  & 0.00     & 0.00              & 0.00   & 0.00   & 918 \\
\toprule
\end{tabular}
\end{table*}
\label{sec:Appendix}


\end{document}